\documentclass[10pt]{article}
\usepackage{etex}
\usepackage[T1]{fontenc}
\usepackage[utf8]{inputenc}
\usepackage{authblk}
 \usepackage[authoryear,round,longnamesfirst]{natbib}

\usepackage{geometry}
\usepackage{graphicx,color,xcolor}
\definecolor{cornellred}{rgb}{0.7, 0.11, 0.11}
\usepackage{hyperref}
\hypersetup{
    colorlinks = true,
    linkcolor=cornellred,
    citecolor=blue,
    linkbordercolor = {white}
}
\usepackage{fancybox}

\usepackage{array,float}
\usepackage{amstext,amssymb,amsmath}

\usepackage{mathrsfs}
\usepackage{dsfont}
\usepackage[noend]{algorithmic}
\usepackage{pstricks, pst-tree, pst-node}
\usepackage{enumitem}
\usepackage{mathtools}

\usepackage{romannum}
\AtBeginDocument{\pagenumbering{arabic}}

\usepackage{hyperref}
\hypersetup{
     colorlinks   = true,
     citecolor    = blue
}
\usepackage{tocloft}

\hypersetup{linkcolor=blue}
\usepackage{amsfonts}
\usepackage{amsmath}
\usepackage{amssymb}
\usepackage{amsthm}
\usepackage{booktabs}
\usepackage[T1]{fontenc}
\usepackage{hyperref}
\usepackage[utf8]{inputenc}
\usepackage{microtype}
\usepackage{nicefrac}
\usepackage{tikz}
\usepackage{url}
\usepackage{subcaption}
\usepackage{colortbl}

\newlength\figureheight 
\newlength\figurewidth 
\usetikzlibrary{calc}

\newtheorem{theorem}{Theorem}[section]

\newtheorem{remark}{Remark}

\newtheoremstyle{restate}{}{}{\itshape}{}{\bfseries}{~(restated).}{.5em}{\thmnote{#3}}
\theoremstyle{restate}

\newcommand\K{\mathcal{K}}
\newcommand\RR{\mathbb{R}}
\newcommand\bfw{\mathbf{w}}
\newcommand\bfx{\mathbf{x}}
\newcommand\np{{\mathbb{NP}}}
\newcommand\bfzero{\mathbf{0}}

\newcommand\pspace{{$\mathbb{PSPACE}$}}
\newcommand\PSPACE{{\mathbb{PSPACE}}}
\newcommand{\C}{{\mathcal{C}}}

\newcommand{\cpath}{$\mathcal{C}$\textsc{-Path}}

\newcommand{\memory}[0]{\{1, \ldots, d\}}
\newcommand{\binary}[0]{\{ -1, +1 \}}
\newcommand{\ubinary}[0]{\{ -1, 0, +1\}}
\newcommand{\rbinary}[1]{\{ -\epsilon_{#1}, +\epsilon_{#1} \}}
\newcommand{\rubinary}[1]{\{ -\epsilon_{#1}, 0, +\epsilon_{#1} \}}
\newcommand{\func}[2]{\texttt{#1}$\left( #2 \right)$}

\newcommand{\fnot}[1]{\func{not}{#1}}
\newcommand{\reset}[1]{\func{reset}{#1}}
\newcommand{\cpy}[2]{\func{copy}{#1, #2}}
\newcommand{\dnand}[3]{\func{destructive\_nand}{#1, #2, #3}}
\newcommand{\inputF}[1]{\func{set\_false\_if\_unset}{#1}}
\newcommand{\setiftrue}[2]{\func{copy\_if\_true}{#1, #2}}

\newcommand{\rreset}[2]{\func{reset}{#1, #2}}
\newcommand{\rcopy}[4]{\func{copy2}{#1, #2, #3, #4}}
\newcommand{\rdnand}[5]{\func{d\_nand}{#1, #2, #3, #4, #5}}
\newcommand{\rinputF}[2]{\func{set\_false\_if\_unset}{#1, #2}}
\newcommand{\rsetiftrue}[3]{\func{copy\_if\_true}{#1, #2, #3}}

\newcommand{\norm}[2]{\left\lVert #1 \right\rVert_{#2}}

\usepackage{color}

\newcommand{\nuke}[1]{}
\begin{document}
\title{On the Computational Power of Online Gradient Descent}

\author[$\dagger$]{Vaggos Chatziafratis}
\author[$\star$]{Tim Roughgarden}
\author[$\ddagger$]{Joshua R. Wang}
\affil[$\dagger$]{Department of Computer Science,  Stanford University}
\affil[$\star$]{Department of Computer Science,  Columbia University}
\affil[$\ddagger$]{Google Research, Mountain View}

%\date{}
%

\maketitle

\begin{abstract}
\normalsize
We prove that the evolution of weight vectors in online gradient
descent can encode arbitrary polynomial-space computations, even in
very simple learning settings.
%the special case of soft-margin support vector machines.  
Our results imply that, under weak complexity-theoretic
assumptions, it is impossible to reason efficiently about the
fine-grained behavior of online gradient descent.
\end{abstract}

\section{Introduction}
\label{sec:intro}
In {\em online convex optimization (OCO)}, an online algorithm picks a
sequence of points $\bfw^1,\bfw^2,...$ from a compact convex set
$\K \subseteq \RR^d$ while an adversary chooses a sequence
$f_1,f_2,...$ of convex cost functions (from $\K$ to $\RR$).  The
online algorithm can choose $\bfw_t$ based on the previously-seen
$f^1, ...,f^{t-1}$ but not later functions; the adversary can choose
$f^t$ based on $\bfw^1, ...,\bfw^t$.  The algorithm incurs a
cost of $f^t(\bfw^t)$ at time~$t$.
Canonically, in a machine learning context, $\K$ is the set of
allowable weight vectors or hypotheses (e.g., vectors with bounded
$\ell_2$-norm), and $f^t$ is induced by a data point $\bfx^t$, a
label $y^t$, and a loss function $\ell$ (e.g., absolute, hinge, or
squared loss) via $f^t(\bfw^t) = \ell(\bfw^t, (\bfx^t, y^t))$.

% -ogd review
%  -general alg
%  -instantiate for ex above

One of the most well-studied algorithms for OCO is {\em online gradient
  descent (OGD)}, which always chooses the point $\bfw^{t+1} := \bfw^{t}
- \eta \cdot \nabla f^{t}(\bfw^t)$~\citep{zinkevich2003online}, projecting
back to $\K$ if necessary.  This algorithm enjoys good guarantees
for OCO problems, such as vanishing regret (see e.g.~\cite{hazanbook}).

%punch line (ogd performs arbitrary pspace computations)

%Soft-margin support vector machines (SVMs) for binary classification
%(see \hyperref[sec:prelim]{Section~\ref{sec:prelim}}), with their strongly %convex loss
%functions,  would seem to be a particularly benign setting in which to
%apply OGD.  And yet:

The main message of this paper is:
\begin{itemize}
\centering
\item [] \textit{OGD captures arbitrary polynomial-space computations,
    even in very simple settings.}
\end{itemize}
For example, this result is true for binary classification using
soft-margin support vector machines (SVMs) or neural networks with one
hidden layer, ReLU activations, and the squared loss function.  (For
even simpler models, like ordinary linear least squares, such a result appears
impossible; see \hyperref[apx:quadratic]{Appendix~\ref{apx:quadratic}}.)

A bit more precisely: for every polynomial-space computation, there is
a sequence of data points $(\bfx^1,y^1),\ldots,(\bfx^T,y^T)$ that have
polynomial bit complexity such that, if these data points are fed to
OGD (specialized to one of the aforementioned settings)
%soft-margin SVMs, with all-zero initial weights)
in this order over and over again, the consequent sequence of weight
vectors simulates the given computation.  \hyperref[f:toy]{Figure~\ref{f:toy}} gives
a cartoon view of what such a simulation looks like.\footnote{Our
  actual simulation in \hyperref[sec:reduction]{Section~\ref{sec:reduction}} and \hyperref[sec:api]{Section~\ref{sec:api}} is similar in spirit to but more complicated than the picture in \hyperref[f:toy]{Figure~\ref{f:toy}}.  For
  example, we use a constant number of OGD updates to simulate each
  circuit gate (not just one), and each weight can take on up to a
  polynomial number of different values.}

\begin{figure}

\centering
\begin{subfigure}{.44\textwidth}
  \begin{center}
\centerline{\includegraphics[scale=0.30]{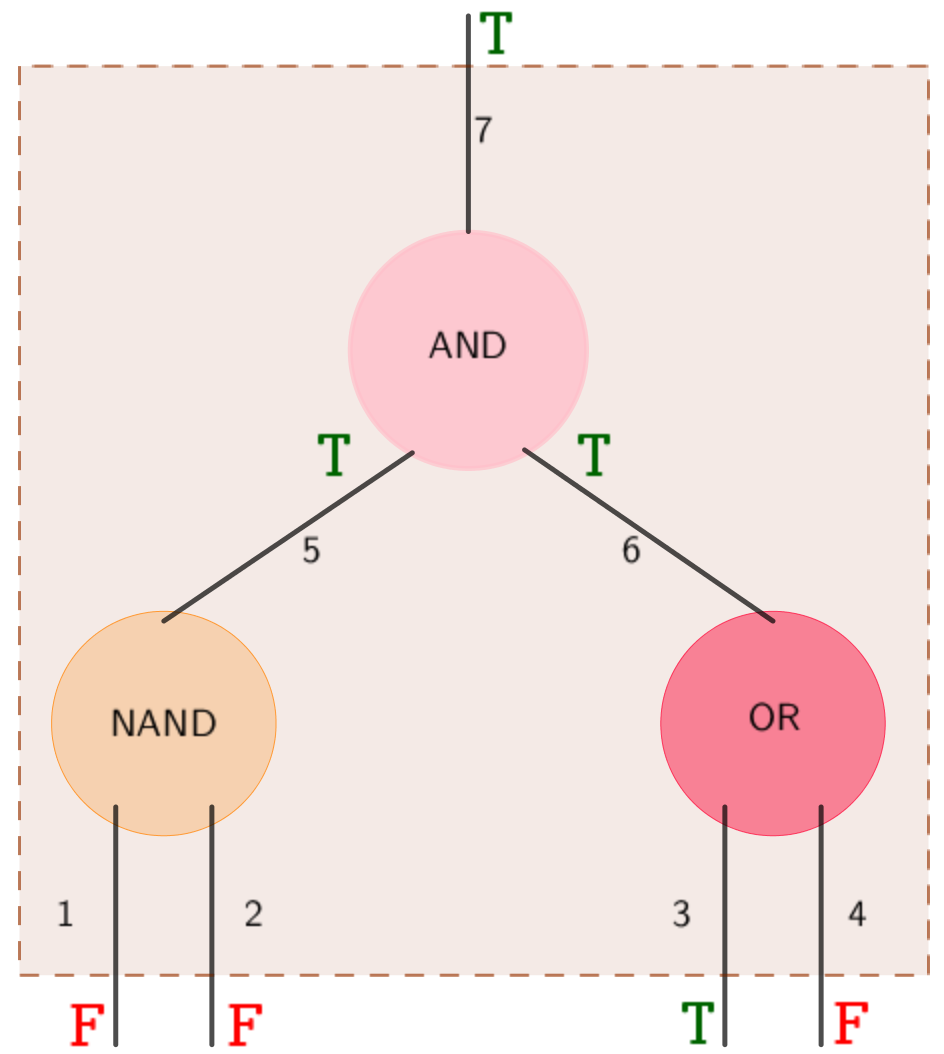}} 
%\caption{The circuit. \label{fig:clusterA}}
\end{center}\end{subfigure}%
\begin{subfigure}{.5\textwidth}
  \begin{center}
  \vspace{1cm}
\centerline{\includegraphics[scale=0.41]{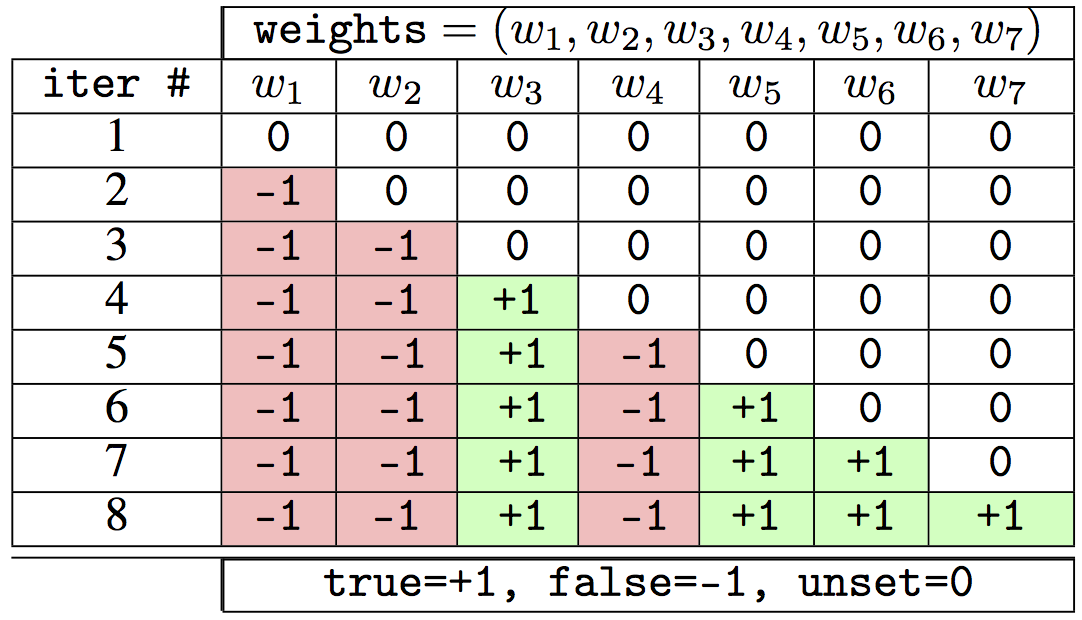}} 
\vspace{0.9cm}

%\caption{The weight vector keeps track of the input/output in the circuit gates. \label{fig:clusterA}}
\end{center}
\end{subfigure}
\caption{Cartoon view of simulating a computation using a sequence
    of weight vectors.  On the left, the evaluation of a Boolean
    circuit on a specific input (with ``T'' and ``F'' indicating which
    inputs and gates evaluate to true and false, respectively).  On the right, a corresponding
    sequence of weight vectors (with updates triggered by a
    carefully chosen data set), with each vector evaluating one more
    gate of the circuit than the previous one.  Weights of $+1, -1$, and
    $0$ indicate that an input has been assigned true, has been assigned
    false, or has not yet been assigned a value, respectively.}
\label{f:toy}
\end{figure}

%interpretation: hard to do fine-grained reasoning about online GD

%ex: problem where input = n data points.
%yes/no question: if you keep running OGD on these over and over, will
%w_1 ever be positive?

%thm: this is pspace-hard.  moral reason: OGD can simulate repeated
%computations by any poly-size circuit.

%interpreting psapce: much bigger than np, for example.

Our simulation implies that, under weak complexity-theoretic
assumptions, it is impossible to reason efficiently about the
fine-grained behavior of OGD.  For example, the following problem is
\pspace-hard\footnote{In fact, for the case where we are promised that
the weights are bounded and only require polynomial bits of precision
(they are so in our constructions), the problem is \pspace-complete, because
we can store the weights in our polynomially-sized memory and can keep a
polynomially-sized timer to check whether we are cycling.
}: given a sequence $(\bfx^1,y^1),\ldots,(\bfx^T,y^T)$
of data points, to be fed into OGD over and over again (in the same
order), with initial weights $\bfw^1 = \bfzero$, does any
weight vector~$\bfw^t$ produced by OGD (with soft-margin SVM updates)
have a positive first coordinate?\footnote{\pspace{} is the set of
  decision problems decidable by a Turing machine that uses space at
  most polynomial in the input size, and it contains problems that are
  believed to
  be very hard (much harder than $\np$-complete).  For example, the
  problem of deciding which player has a winning strategy in chess
  (for a suitable asymptotic generalization of chess) belongs to (and
  is complete for) \pspace~(\cite{Storer83}).}

In the case of soft-margin SVMs, for
%In
the instances produced by our reduction, the optimal point in
hindsight converges over time to a single point $\bfw^*$ (the
regularized ERM solution for the initial data set), and the well-known
regret guarantees for OGD imply that its iterates grow close to
$\bfw^*$ (in objective function value and, by strong convexity, in
distance as well).  
%This is a remarkable co-existence of approximate
%convergence with endlessly complex variation.
Viewed from afar, OGD is nearly converging; viewed up close, it
exhibits astonishing complexity.

%Because our constructions apply to OGD with the same sequence of
%points fed in over and over, the best point in hindsight

%SGD interpretation: in practice, SGD often processes the points in a
%fixed and atbritrary order.  (rather than using randomization.)  our
%results thus carry over to this way of implementing SGD, for a
%worst-case ordering of the data points.

Our results have similar implications for a common-in-practice variant
of stochastic gradient descent (SGD), where every epoch performs a
single pass over the data points, in a fixed but arbitrary order.  
Our work implies that this variant of SGD can also simulate arbitrary
\pspace{} computations (when the data points and their ordering can be
chosen adversarially).

\subsection{Related Work}

%oco
%regret guarantees
% -why our results don't contradict them
%simplex, k-means, lemke-howson, limit cycles, etc.
%known bad examples for online GD, at least in the non-convex case?

There are a number of excellent sources for further background on OCO,
OGD, and SVMs; see e.g.~\cite{hazanbook,SSBDbook}.  We use only
classical concepts from complexity theory, covered e.g.~in~\cite{sipserbook}.

There is a long history of \pspace-completeness results for reasoning
about iterative algorithms.  For example,
\pspace-completeness results were proved for computing the final
outcome of local search~\citep{JPY88} and other path-following-type
algorithms~\citep{GPS13}.
For a more recent example that concerns finding a limit cycle of
certain dynamical systems, see~\cite{vishnoi}.

This paper is most closely related to a line of work showing that
certain widely used algorithms inadvertently solve much harder
problems than what they were originally designed for.  For example,
%Adler et al.~
\cite{adler},
%Disser and Skutella~
\cite{DS15},
%and Fearnley and Savani~
and \cite{FS16}
show how to efficiently embed an instance of a hard problem into a
linear program
%an
%instance of the ($\mathbb{NP}$-complete) Partition problem into a linear
%program
so that the trajectory of the simplex method immediately
reveals the answer to the instance.
%This line of work was expanded on
%by Adler et al.~\cite{adler} and Fearnley and Savani~\cite{FS16}, who
%exhibited pivot rules with which the simplex method can solve
%\pspace-complete problems.
%Roughgarden and Wang~
\cite{RW16} proved an
analogous \pspace-completeness result for Lloyd's $k$-means
algorithm.

%Similar

%All of these results echo earlier works that prove the
%\pspace-completeness of computing the final outcome of local
%search~\cite{JPY88}, dynamical systems~\cite{vishnoi}, and other
%path-following-type algorithms~\cite{GPS13}.

More distantly related are previous works that treat stochastic
gradient descent as a dynamical system and then show that the system
is complex in some sense.  Examples include~\cite{chaos},
who provide empirical evidence of chaotic behavior, and~\cite{CS2018},
who show that, for DNN training, SGD can converge to stable limit
cycles.  We are not aware of any previous works that take a
computational complexity-based approach to the problem.

%%% Local Variables:
%%% mode: latex
%%% TeX-master: "../main"
%%% End:

\section{Preliminaries}
\label{sec:prelim}
\subsection{Soft-Margin SVMs}

%-hinge loss
%regularizer vs. not

We begin with the following special case of OCO, corresponding to soft-margin
support vector machines (SVMs) under a hinge loss.\footnote{Neural
networks with ReLU activations and squared loss are discussed in \hyperref[apx:additional]{Appendix~\ref{apx:additional}}.}
For some fixed regularization parameter $\lambda$, every cost function $f^t$
will have the form 
\[
\ell_{hinge}(\bfw^t,(\bfx^t,y^t)) + \tfrac{\lambda}{2} \|\bfw^t\|_2^2
\]
for some data point $\bfx^t \in \RR^d$ and label $y^t \in \{-1,+1\}$,
where the hinge loss is defined as $\ell_{hinge}(\bfw^t,(\bfx^t,y^t))
= \max \{ 0, 1 - y^t (\bfw^t \cdot \bfx^t) \}$.\footnote{For
  simplicity, we have omitted the bias term here; see also
  \hyperref[ss:bias]{Section~\ref{ss:bias}}.}
In this case, the weight updates in OGD have a special form (where
$\eta$ is the step size):
\[
\bfw^{t+1} = (1-\lambda\eta) \bfw^{t} + \eta \cdot \left\{ 
\begin{array}{cl}
y^t(\bfx^t) & \text{if $y^t (\bfw^t \cdot \bfx^t) < 1$}\\
0 & \text{if $y^t (\bfw^t \cdot \bfx^t) > 1$.}
\end{array}
\right.
\]

\subsection{Complexity Theory Background}

%defn of pspace
%defn of circuit-path
%relation between pspace and other classes.

A decision problem $L \subseteq \{0,1\}^*$ is in the class \pspace{} if
and only if there exists a Turing machine $M$ and a polynomial
function $p(\cdot)$ such that, for every $n$-bit string $z$, $M$
correctly decides whether or not $z$ is in $L$ while using space at
most $p(n)$.

\pspace{} is obviously at least as big as $\mathbb{P}$, the class of
polynomial-time-decidable decision problems (it takes $s$ operations
to use up $s$ tape cells).  It also contains every problem in $\np$
(just try all possible polynomial-length witnesses, reusing space for
each computation), $co$-$\mathbb{NP}$ (for the same reason), the
entire polynomial hierarchy, and more.  A problem $L$ is {\em
\pspace-hard} if every problem in \pspace{} polynomial-time
reduces to it, and {\em \pspace-complete} if additionally $L$ belongs
to \pspace.  While the current state of knowledge does not rule out
$\mathbb{P = PSPACE}$ (which would be even more surprising than
$\mathbb{P=NP}$), the widespread belief is that $\mathbb{PSPACE}$
contains many problems that are intrinsically computationally
difficult (like the aforementioned chess example).  Thus a problem
that is complete (or hard) for \pspace{} would seem to be very hard
indeed.

Our main reduction is from the \cpath{} problem.  In this problem, the
input is (an encoding of) a Boolean circuit~$\mathcal{C}$ with $n$ inputs, $n$
outputs, and gates of fan-in $2$; and a target $n$-bit string $s^*$.  The goal is
to decide whether or not the repeated application of $\mathcal{C}$ to the
all-false string ever produces the output~$s^*$.  This problem is
\pspace-complete (see \cite{adler}), and in this sense every
polynomial-space computation is just a thinly disguised instance of
\cpath{}.

\section{$\PSPACE$-Hardness Reduction}
\label{sec:reduction}

\newcommand{\specx}{\bot}
\newcommand{\specy}{\Box}
\newcommand{\specz}{\Diamond}

In this section, we present our main reduction from the \cpath{} problem. Our reduction uses several types of gadgets, which are organized into an API in \hyperref[ss:apieasy]{Subsection~\ref{ss:apieasy}}.
%Subsection~\ref{ss:apieasy}. 

The implementation of two gadgets is given in \hyperref[sec:api]{Section~\ref{sec:api}} and the remaining implementations can be found in \hyperref[sec:appendix]{Appendix~\ref{sec:appendix}}. After presenting the API, this section concludes by showing how the reduction can be performed using the API.

\subsection{Simplifying Assumptions}

For this section, we make a couple of simplifying assumptions to showcase the main technical ideas used in our proof. We later show how to extend the proof to remove these assumptions in \hyperref[sec:extensions]{Section~\ref{sec:extensions}}. Our simplifying assumptions are:
\begin{enumerate}
  \item[(i)] There is no bias term, i.e. $b$ is fixed to $0$.
  \item[(ii)] The learning rate $\eta$ is fixed to $1$.
  \item[(iii)] The loss function is not regularized, i.e. $\lambda = 0$.
\end{enumerate}

\subsection{API for Reduction Gadgets}
\label{ss:apieasy}

We use a number of gadgets to encode an instance of \cpath{} into
training examples for OGD. The high level plan is to use the weights
$\bfw^t$ to encode boolean values in our circuit. A weight of $+1$
will represent a true bit, while a weight of $-1$ will represent a
false bit. Additionally, we use a weight of $0$ to represent a bit
that we have not yet computed (which we refer to as ``unset'').
For example, our simplest gadget is \reset{i_1}, which
takes the index of a weight that is set to either +1 or -1, 
and provides a sequence of training examples that causes that weight
to update to 0 (thus unsetting the bit). Our next simplest gadget
is \fnot{i_1}, which takes the index of a weight that is set to either
+1 or -1, and provides a sequence of training examples that causes the
weight to update to -1 or +1, respectively (thus setting it to the not of
itself). Note that our main reduction does not use the \texttt{not} gadget
directly, but it serves as a subgadget for our other gadgets and is also
useful for performing other reductions.

It is well known that every $\{\pm 1\}$ Boolean circuit can be efficiently
converted into a circuit that only has NAND gates (where the output
is~$-1$ if both inputs are~$+1$, and $+1$ otherwise), and so we focus on such
circuits.
%We use a standard technique to assume without loss of generality that
%the circuit consists only of NAND gates.
%As a consequence, w
We would like a gadget that takes two true/false bits and an unset bit
and writes the NAND of the first two into the third. Unfortunately,
the nature of the weight updates makes it difficult to
implement NAND directly. As a result, we instead use two smaller gadgets that can
together be used to compute a NAND. The bulk of the work is done by
\dnand{i_1}{i_2}{i_3}, which performs the above but has the
unfortunate side-effect of unsetting the first two bits. As a result,
we need a way to increase the number of copies we have of a boolean
value. The \cpy{i_1}{i_2} gadget takes a true/false bit and an unset
bit and writes the former into the latter. Taken together, we can
compute NAND by copying our two bits of interest and then using the
copies to compute the NAND.

Our next gadget allows the starting weights $\bfw^0$ to be the
all-zeroes vector. The gadget \\ \inputF{i_1} takes a weight that may
correspond to either a true/false bit or to an unset bit. If the
weight is already true/false, it does nothing. Otherwise, it takes the
unset bit and writes false into it. 
%This gadget because during the
%first pass over the training data, we would like to write the starting
%input to the \cpath{} problem, but in later passes we would like the
%input to the circuit to be the output from the previous pass.

Finally, we have a simple gadget for the purpose of presenting a
concrete \pspace-hard decision problem about the OGD process. The
question we aim for is, does any weight vector $\bfw^t$ produced by
OGD (with soft-margin SVM updates) have a positive first coordinate?
Correspondingly, the \texttt{set\_if\_true} gadget takes a true/false bit and a
zero-weight coordinate (intended to be the first coordinate). If the
first bit is true, this gadget gives the zero-weight coordinate a
weight of $+1$. If the first bit is false, this gadget leaves the
zero-weight coordinate completely untouched, even in intermediate
steps between its training examples. This property is not present in
the implementation of our other gadgets, so this will be the only
gadget that we use to modify the first coordinate.

This API is formally specified in \hyperref[tab:public-api]{Table~\ref{tab:public-api}}.

\begin{table}
\caption{Public API}
\label{tab:public-api}
\centering
\begin{tabular}{@{}lll@{}}
  \toprule
  \textbf{Function} & \textbf{Precondition(s)} & \textbf{Description} \\ \midrule
  % reset(i_1)
  \reset{i_1} &
  $i_1 \in \memory$ &
  $w_{i_1} \leftarrow 0$ \\
  (for implementation, see \hyperref[tab:reset]{Table~\ref{tab:reset}})&
  $w_{i_1} \in \binary$ &
  \\ \midrule
  % not(i_1)
  \fnot{i_1} &
  $i_1 \in \memory$ &
  $w_{i_1} \leftarrow \text{NOT}(w_{i_1})$ \\
  (for implementation, see \hyperref[tab:not]{Table~\ref{tab:not}}) &
  $w_{i_1} \in \binary$ &
  \\ \midrule
  % copy(i_1, i_2)
  \cpy{i_1}{i_2} &
  $i_1, i_2 \in \memory$ &
  $w_{i_2} \leftarrow w_{i_1}$ \\
  (for implementation, see  \hyperref[tab:cpy]{Table~\ref{tab:cpy}})&
  $w_{i_1} \in \binary$ &
  \\
  &
  $w_{i_2} = 0$ &
  \\ \midrule
  % destructive_nand(i_1, i_2, i_3)
  \dnand{i_1}{i_2}{i_3} &
  $i_1, i_2, i_3 \in \memory$ &
  $w_{i_3} \leftarrow \text{NAND}(w_{i_1}, w_{i_2})$ \\
  (for implementation, see  \hyperref[tab:dnand]{Table~\ref{tab:dnand}})&
  $w_{i_1} \in \binary$ &
  $w_{i_1} \leftarrow 0$ \\
  &
  $w_{i_2} \in \binary$ &
  $w_{i_2} \leftarrow 0$ \\
  &
  $w_{i_3} = 0$ &
  \\ \midrule
  % initialize_false(i_1)
  \inputF{i_1} &
  $i_1 \in \memory$ &
  If $w_{i_1} == 0$, $w_{i_1} \leftarrow -1$ \\
  (for implementation, see  \hyperref[tab:inputF]{Table~\ref{tab:inputF}})&
  $w_{i_1} \in \ubinary$ &
  \\ \midrule
  % set_if_true(i_1, i_2)
  \setiftrue{i_1}{i_2} &
  $i_1, i_2 \in \memory$ &
  If $w_{i_1} > 0$, $w_{i_2} \leftarrow +1$ \\
  (for implementation, see  \hyperref[tab:setiftrue]{Table~\ref{tab:setiftrue}})&
  $w_{i_1} \in \binary$ &
  If $w_{i_1} < 0$, $w_{i_2}$ remains at $0$ \\
  &
  $w_{i_2} = 0$ &
  (including in intermediate steps) \\
  \bottomrule
\end{tabular}
\end{table}

\subsection{Performing the Reduction using the API}

We now show how to use our API to transform an instance of the \cpath{} problem into a set of training examples for a soft-margin SVM that is being optimized by OGD.

\begin{theorem}\label{th:main}
There is a reduction which, given a circuit $\C$ and a target binary string $s^*$, produces a set of training examples for OGD (with soft-margin SVM updates) such that repeated application of $\C$ to the all-false string eventually produces the string $s^*$ if and only if OGD beginning with the all-zeroes weight vector and repeatedly fed this set of training examples (in the same order)  eventually produces a weight vector $\bfw^t$ with positive first coordinate.
\end{theorem}

\begin{proof}
  Our reduction begins by converting $\C$ into a more complex circuit
  $\C'$.  First, we assume that $\C$ has only NAND gates (see above).
  Next, we augment our circuit with an additional input/output bit,
  intended to track if the current output is $s^*$. The circuit $\C'$
  ignores its additional input bit, and its additional output bit is
  true if the original output bits are $s^*$ and false
  otherwise. These transformations keep the size of $\C'$ polynomial
  in the input/output size.
  
Let $n$ denote the input/output size of $\C'$ and let $m$ denote the
number of gates in $\C'$. Our reduction produces training examples for
an SVM with a $d$-dimensional weight vector, where $d = n + m + 3$. We denote the
first three indices for this weight vector using $\specx$, $\specy$,
and $\specz$: notably, $\specx$ denotes the first coordinate whose
weight should remain zero unless the input to the \cpath{} problem
should be accepted. We denote the next $n$ indices $1, \ldots, n$ and
associate each with an input bit. We denote the last $m$ indices $n+1,
\ldots, n+m$ and associate them with gates of $\C'$, in some
topological order.

We begin with an empty training set. Each time we call a function from
our API (which can be found in \hyperref[tab:public-api]{Table~\ref{tab:public-api}}), we append
its training examples to the \emph{end} of our training set. We now
give the construction, and then finish the proof by proving the
resulting set of training examples has the desired property. Our
construction proceeds in five phases.

In the first phase of our reduction, we set the starting input for the
\cpath{} problem. We iterate in order through $i = 1, 2, \ldots,
n$. In iteration $i$, we call \inputF{i}.
  
In the second phase of our reduction, we simulate the computation of the
circuit $\C'$. We iterate in order through
$i = n+1, n+2, \ldots, n+m$. In iteration $i$, we examine the NAND
gate in $\C'$ associated with $i$. Suppose its inputs are associated
with indices $i_1$ and $i_2$. We call \cpy{i_1}{\specy},
\cpy{i_2}{\specz}, \dnand{\specy}{\specz}{i} in that order.
  
  In the third phase of our reduction, we check if we have found $s^*$. Let the additional output bit of $\C'$ be at index $i_1$. We call \setiftrue{i_1}{\specx}.
  
  In the fourth phase of our reduction, we copy the output of the circuit back to the input. We iterate in order through $i = 1, 2, \ldots, n$. In iteration $i$, let the $i^{th}$ output bit of $\C'$ correspond to the gate associated with index $i_1$. We call \reset{i} and \cpy{i_1}{i}, in that order.
  
  In the fifth phase of our reduction, we reset the circuit for the next round of computation. We iterate in order through $i = n+1, n+2, \ldots, n+m$. In iteration $i$, we call \reset{i}.
  
  We now explain why the resulting training data has the desired property. Let's consider what OGD does in (i) the first pass over the training data and (ii) in later passes over the data. We begin with case (i). Before the first phase of our reduction, all weights are zero, corresponding to unset bits. The first phase of our reduction hence sets the weights at indices $1, \ldots, n$ to correspond to an all-false input. The second phase of our reduction then computes the appropriate output for each gate and sets it. Note that it is important we proceeded in topological order, so that the inputs of a NAND gate are set before we attempt to compute its output. The third phase of our reduction checks if we have found $s^*$, and if the $\specx$ weight gets set to a positive coordinate, this implies that $\C$ immediately produced $s^*$ when applied to the all-false string. The fourth phase of our reduction unsets the weights at indices $1, \ldots, n$ and then copies the output of $\C'$ into them. The fifth phase of our reduction then unsets the weights at indices $n+1, \ldots, n+m$.
  
  If we are continuing after this first pass, then the weights at indices $\specx$, $\specy$, $\specz$, and $n + 1, \ldots, n + m$ are unset while the weights at indices $1, \ldots, n$ are set to the next circuit input. We now analyze case (ii), assuming it also leaves the weights in this state after each pass. In the first phase of our reduction, nothing happens because the input is already set. The second through fifth phases of our reduction then proceed exactly as in case (i), computing the circuit based on this input, checking if we found $s^*$, copying the output to the input, and resetting the circuit for another round of computation. As a result, we again arrive at a state where the weights at indices $\specx$, $\specy$, $\specz$, and $n + 1, \ldots, n + m$ are unset while the weights at indices $1, \ldots, n$ are set to the next circuit input.
  
  In other words, repeatedly passing over our training data causes OGD to simulate the repeated application of $\C$, as desired. By construction, our first coordinate $\specx$ has a positive weight if and only if our simulated $\C$ computation manages to find $s^*$. This completes the proof.
\end{proof}

\begin{remark}
Although our decision question about OGD asked whether the first coordinate ever became positive, our reduction technique is flexible enough to result in many possible decision questions. For example, we might ask if OGD, after a single complete pass over the training examples, winds up producing the same weight vector $\bfw^t$ that it had produced immediately preceding the complete pass (since $\C$ may be rewired so that its only stationary point is $s^*$). As another example, with a simple modification of our \setiftrue{i_1}{i_2} gadget to place a high value into $w_{i_2}$, we could ask whether OGD ever produces a weight vector $\bfw^t$ with norm above some threshold.
\end{remark}
\section{API Implementation}
\label{sec:api}
Now that we have described at a high level how to simulate the circuit computation using OGD updates, we proceed by giving the technical details of the implementation for each gadget operation on the circuit bits: $\texttt{reset}, \texttt{not}, \texttt{copy}, \texttt{destructive\_nand}, \texttt{input\_false},\texttt{set\_if\_true}$. Note that in all of our constructions the training examples required are extremely sparse; each construction involves at most $3$ non-zero coordinates.

\subsection{Implementation of \reset{i_1}}
The \texttt{reset} gadget (see~\hyperref[tab:reset]{Table~\ref{tab:reset}}) takes as input one index $i_1$ and resets the corresponding weight coordinate to zero independent of what this coordinate used to be (either $-1$ or $+1$). The plan is to collapse the two possible states into a single state, then force the weight coordinate to zero.

Since this is our first gadget, we will need to do some legwork and write down the gradients involved in an update. For a datapoint $(\bfx,y)$, the hinge loss function is: $\ell_{hinge}(\bfw, \bfx, y) = \max \{0,1 - y \bfw\cdot \bfx)\}$ and the update is:
\[
  \frac{\partial \ell_{hinge}(\bfw, \bfx, y)}{\partial w_i}
  =
  \begin{cases}
    -y x_i & \text{if } y \bfw\cdot \bfx < 1 \\
    0 & \text{if } y \bfw\cdot \bfx > 1
  \end{cases}
\]

\begin{table}
\centering
\begin{minipage}{.5\textwidth}
\caption{Training data for \reset{i_1}.}
\label{tab:reset}
\centering
\begin{tabular}{@{}ccr@{}r@{$) \to ($}r@{$)$}@{}}
  \toprule
  $x_{i_1}$ & $y$ & \multicolumn{3}{c}{Effect on $(w_{i_1})$} \\ \midrule
  $-2$ & $1$ & $($ & $-1$ & $-1$ \\
  \multicolumn{2}{c}{} & $($ & $1$ & $-1$ \\ \midrule
  $1$ & $1$ & $($ & $-1$ & $0$ \\
  \multicolumn{2}{c}{(add trick)} & $($ & $-1$ & $0$ \\
  \bottomrule
  \vspace{0.31cm}
\end{tabular}
\end{minipage}%
\begin{minipage}{.5\textwidth}
\caption{Training data for \fnot{i_1}.}
\label{tab:not}
\centering
\begin{tabular}{@{}ccr@{}r@{$) \to ($}r@{$)$}@{}}
  \toprule
  $x_{i_1}$ & $y$ & \multicolumn{3}{c}{Effect on $(w_{i_1})$} \\ \midrule
  $4$ & $1$ & $($ & $-1$ & $3$ \\
  \multicolumn{2}{c}{} & $($ & $1$ & $1$ \\ \midrule
  $-2$ & $1$ & $($ & $3$ & $1$ \\
  \multicolumn{2}{c}{(add trick)} & $($ & $1$ & $-1$ \\
  \bottomrule
  \vspace{0.31cm}
\end{tabular}
\end{minipage}
\end{table}

Following our plan, we don't know $w_{i_1}$ but want to collapse the two possible states to a single state. What is an appropriate training example that will allow us to do so? Consider the first training example listed in \hyperref[tab:reset]{Table~\ref{tab:reset}}; we have that $x_{i_1} = -2$, $\bfx$ is zero on the remainder of its coordinates, and $y = +1$. There are two cases to consider when we apply this training example.
\begin{itemize}
\item In the case of $w_{i_1}= -1$, we have $y \bfw\cdot \bfx = (-1)(-2) > 1$ and so there is no update since the gradient of the hinge loss is zero. Hence $w_{i_1}$ remains $-1$.
\item If $w_{i_1}= +1$, we have $y \bfw\cdot \bfx = (+1)(-2) < 1$, and so there is an update. After this update we get: $w_{i_1}\leftarrow w_{i_1} + (+1)(-2)\implies w_{i_1}\leftarrow -1$, as desired. 
\end{itemize}

We have now successfully collapsed into a single state. The next step of our plan was to force the weight coordinate to zero; we want to add $+1$ to $-1$. As it turns out, adding a positive amount to a negative weight (or a negative amount to a positive weight) is easy, and can be done in a single training example. The signs work out so that we can ignore the hinge criterion and choose values that would result in the correct update, and the hinge criterion is naturally satisfied. In the implementation of other gadgets, we will refer to this as the add trick.

Consider the second training example listed in \hyperref[tab:reset]{Table~\ref{tab:reset}}; we have that $x_{i_1} = +1$, $\bfx$ is zero on the remainder of its coordinates, and $y = +1$. Since we know that $w_{i_1}= -1$, we have that $y \bfw\cdot \bfx = (+1)(-1) < 1$ and so there is an update. After this update we get: $w_{i_1}\leftarrow w_{i_1} + (+1)(+1)\implies w_{i_1}\leftarrow 0$, as desired.

\subsection{Implementation of \fnot{i_1}}

The \texttt{not} gadget (see~\hyperref[tab:not]{Table~\ref{tab:not}}) takes as input one index $i_1$ and negates the corresponding weight coordinate. The gadget construction plan is to first swap the roles of high state/low state while maintaining a gap of two, then lower states to the proper values.

Following our plan, we don't know $w_{i_1}$ but want to reverse the order of the states. The more important training example is the first training example listed in \hyperref[tab:not]{Table~\ref{tab:not}}; we have that $x_{i_1} = +4$, $\bfx$ is zero on the remaining coordinates, and the label is $+1$.
\begin{itemize}
\item If $w_{i_1}= -1$, we have $y \bfw\cdot \bfx = (-1)(+4) < 1$, and so there is an update. After this update we get: $w_{i_1}\leftarrow w_{i_1} + (+1)(+4)\implies w_{i_1}\leftarrow +3$.
\item In the case of $w_{i_1}= +1$, we have $y \bfw\cdot \bfx = (+1)(+4) > 1$ and so there is no update since the gradient of the hinge loss is zero. Hence $w_{i_1}$ remains $+1$.
\end{itemize}

Hence we have swapped the low-value state with the high-value state, while maintaining a difference of two between the two states. The second training example is the same add trick that we used before; we add $-2$ to two possible (positive) states, resulting in our desired final values.

All the necessary technical details on how one can implement $\texttt{copy}, \texttt{destructive\_nand}, \\ \texttt{set\_false\_if\_unset}$ and $\texttt{copy\_if\_true}$ are provided in \hyperref[sec:appendix]{Appendix~\ref{sec:appendix}}.
\section{Extensions}
\label{sec:extensions}

In this section, we give extensions to our proof techniques to remove the assumptions we made in \hyperref[sec:reduction]{Section~\ref{sec:reduction}}.

\subsection{Handling a Bias Term}
\label{ss:bias}

In this subsection, we show how to remove assumption (i) and handle an SVM bias term. With the bias term added back in, the loss function is now:
\begin{align*}
  \ell_{hinge}(\bfw, b, \bfx, y)
  &=
  \max \{0, 1 - y (\bfw \cdot \bfx - b)\}
\end{align*}

Using a standard trick, we can simulate this bias term by adding an extra dimension $b_1$ and insisting that $x_{b_1} = -1$ for every training point; the corresponding $w_{b_1}$ entry plays the role of $b$. We now explain how to modify the reduction to follow the restriction that $x_{b_1} = -1$ for every training point.

The key insight is that if we can ensure that the value of this bias term is $w_{b_1} = 0$ immediately preceding every training example from the base construction, then $y (\bfw \cdot \bfx)$ will remain the same and the base construction will proceed as before. The problem is that whenever a base construction training example is in the first case for the derivative (namely $y (\bfw \cdot \bfx) < 1$), this will result in an update to $w_{b_1}$. Since every base construction training example chooses $y = +1$, we know the first case causes $w_{b_1}$ to be updated from $0$ to $-1$. We need to insert an additional training example to correct it back to $0$. To complicate matters further, we sometimes don't know whether we are in the first or second case for the derivative, so we don't know whether $w_{b_1}$ has remained at $0$ or has been altered to $-1$. We need to provide a gadget such that for either case, $w_{b_1}$ is corrected to $0$.

In order to avoid falling on the border of the hinge loss function ($y (\bfw \cdot \bfx) = 1$), we will be using \emph{two} mirrored bias terms. In other words, we add two extra dimensions, $b_1$ and $b_2$ and insist that $x_{b_1} = x_{b_2} = -1$ for every training point. We ensure that $w_{b_1} = w_{b_2} = 0$ before every base construction training example. Since they always have the same weight, the two points always receive the same update, and the situtation is now that either (i) they both remained at $0$ or (ii) they both were altered to $-1$. We would like to correct them both to $0$.

\begin{table}
\caption{Training data to correct the bias term.}
\label{tab:bias}
\centering
\begin{tabular}{@{}cccr@{}r@{, }r@{$) \to ($}r@{, }r@{$)$}@{}}
  \toprule
  $x_{b_1}$ & $x_{b_2}$ & $y$ & \multicolumn{5}{c}{Effect on $(w_{b_1}, w_{b_2})$} \\ \midrule
  $-1$ & $-1$ & $1$ & $($ & $-1$ & $-1$ & $-1$ & $-1$ \\
  \multicolumn{3}{c}{} & $($ & $0$ & $0$ & $-1$ & $-1$ \\ \midrule
  $-1$ & $-1$ & $-1$ & $($ & $-1$ & $-1$ & $0$ & $0$ \\
  \multicolumn{3}{c}{} & $($ & $-1$ & $-1$ & $0$ & $0$ \\
  \bottomrule
\end{tabular}
\end{table}

The two training examples that implement this behavior can be found in \hyperref[tab:bias]{Table~\ref{tab:bias}}. The first training example combines cases by transforming case (i) into case (ii) and resulting in no updates when in case (ii). The second training example then resets both values to $0$. To fix the base construction, we insert this gadget immediately after every base training example. As stated previously, this guarantees that $w_{b_1} = 0$ immediately before every base construction training example, which thus proceeds in the same fashion.

\subsection{Handling a Fixed Learning Rate}
\label{ss:eta}

In this subsection, we show how to remove our assumption that the learning rate $\eta = 1$. Suppose we have some other step size $\eta$, possibly a function of $T$, the total number of steps to run OGD. We perform our reduction from \cpath{} as before, pretending that $\eta = 1$. This yields a value for $T$, which we can then use to determine $\eta(T)$.

We then scale all training vectors $\bfx$ (but not labels $y$) by $\frac{1}{\sqrt{\eta}}$. We claim that our analysis holds when the weight vectors $\bfw$ are scaled by $\sqrt{\eta}$. To see why, we reconsider the updates performed by OGD. First, consider the gradient terms:
\begin{align*}
  \frac{\partial \ell_{hinge}(\bfw, \bfx, y)}{\partial w_i}
  &=
  \begin{cases}
    -y x_i & \text{if } y (\bfw \cdot \bfx) < 1 \\
    0 & \text{if } y (\bfw \cdot \bfx) > 1
  \end{cases}
\end{align*}

Notice that the scaling of $\bfx$ and the scaling of $\bfw$ cancel out when computing $\bfw \cdot \bfx$, so we stay in the same case. Since $\bfx$ was scaled by $\frac{1}{\sqrt{\eta}}$, our gradients scale by that amount as well. However, since the updates performed are $\eta$ times the new gradient, the net scaling of updates to $\bfw$ is by a factor of $\sqrt{\eta}$. Since our analysis of $\bfw$ is scaled up by exactly this amount as well, $\bfw$ is updated as we previously reasoned.

As an aside, one common use case is annealing the learning rate, e.g. $\eta_t = 1/\sqrt{t}$. For this case, it is possible to use our machinery to perform a circuit to OGD reduction, but the result would be that determining the exact result of OGD after it is fed a series of examples once (not repeatedly) is $P$-complete (computable in polynomial time, but probably not parallelizable). The issue is that different passes over the training data would be performed at different scales, but we can still get some complexity out of a single pass.

\subsection{Handling a Regularizer}

In this subsection, we discuss how to handle a regularization parameter $\lambda$ which is not too large. Consider the hinge loss objective with a regularizer:
\begin{align*}
  \ell_{reg}(\bfw, \bfx, y)
  &=
  \max \{0, 1 - y (\bfw \cdot \bfx)\} + \tfrac{\lambda}{2} \norm{\bfw}{2}^2 \\
  \frac{\partial   \ell_{reg}(\bfw, \bfx, y)}{\partial w_i}
  &=
  \begin{cases}
    -y x_i & \text{if } y (\bfw \cdot \bfx) < 1 \\
    0 & \text{if } y (\bfw \cdot \bfx) > 1
  \end{cases} \\
  &\hphantom{=}
  +
  \lambda w_i
\end{align*}

Conceptually, the regularizer causes our weights to slowly decay over time. In particular, this new $\lambda w_i$ term in the gradient means that weights decay by $\alpha = (1 -\lambda)$ at each step. We assume that this decay rate is not too fast: $\alpha \in \left(\frac{1}{\sqrt{2}}, 1 \right)$. Equivalently, $\lambda \in \left( 0, 1 - \frac{1}{\sqrt{2}} \right)$. Due to this decay, we will no longer be able to maintain the association that a true bit is $+1$, a false bit is $-1$, and an unset bit is $0$. Instead, for each weight index $i$ the reduction will need to maintain a counter $\epsilon_i$ which represents the current magnitude of any true/false bit being stored in that weight variable $w_i$. A true bit will be $+\epsilon_i$, a false bit will be $-\epsilon_i$, and an unset bit will still be $0$. After each training example it adds, the reduction should multiply each counter $\epsilon_i$ by $\alpha$.

Correspondingly, our API will need to grow more complex as well. The new API, the modified reduction which uses it, and the formal implementation can all be found in \hyperref[sec:regularizer]{Appendix~\ref{sec:regularizer}}.

\newcommand{\nand}[2]{\text{NAND}\left(#1, #2\right)}
\newpage
\bibliographystyle{alpha}
\bibliography{main}
\newpage
\appendix

\section{Barrier for Quadratic Models}
\label{apx:quadratic}

In this appendix, we explain why our reductions cannot go through for a large class of models. This class includes the method of least squares, in which the loss function for the current choice of weights $\bfw^t$ and a point $(\bfx^t, y^t)$ is given by:
\[
  \ell_{LS}(\bfw^t, (\bfx^t, y^t)) = (y^t - \bfw^t \cdot \bfx^t)^2
\]

More specifically, this barrier applies to any model where the loss function is quadratic in the weights, i.e. of the following form.
\[
  \ell(\bfw^t, (\bfx^t, y^t)) = \sum_{i=1}^d \sum_{j=1}^d \alpha_{i, j}(\bfx^t, y^t) w_i w_j
      + \sum_{i=1}^d \beta_{i}(\bfx^t, y^t) w_i
      + \gamma(\bfx^t, y^t)
\]
Note that the quadratic coefficients $\alpha, \beta, \gamma$ may be \emph{arbitrary functions} of the training points, and without loss of generality we consider the coefficients $\alpha$ to be symmetrized so that $\alpha_{i, j} = \alpha_{j, i}$.

The key point about such functions is that the gradient update with respect to point $(\bfx^t, y^t)$ is a linear transformation of the weights. In particular, notice that the derivative with respect to the $k^{th}$ weight is:
\[
  \frac{\partial \ell}{\partial w_k} = 2 \sum_{i=1}^d \alpha_{i, k}(\bfx^t, y^t) w_i + \beta_k(\bfx^t y^t)
\]

Hence an OGD with fixed step size $\eta$ will have the form:
\[
  w^{t+1}_k = w^t_k - \eta \left[ 2 \sum_{i=1}^d \alpha_{i, k}(\bfx^t, y^t) w_i + \beta_k(\bfx^t y^t) \right]
\]

We can hence write our update as a matrix-vector product if we augment our weight vector with a one:
\[
  \begin{bmatrix}
    w^{t+1}_1 \\
    w^{t+1}_2 \\
    \vdots \\
    w^{t+1}_d \\
    1
  \end{bmatrix}
  =
  \underbrace{
  \left(
  I_{d+1} - \eta
  \begin{bmatrix}
    2 \alpha_{1, 1} & 2 \alpha_{1, 2} & \hdots & 2 \alpha_{1, d} & \beta_1 \\
    2 \alpha_{2, 1} & 2 \alpha_{2, 2} & \hdots & 2 \alpha_{2, d} & \beta_2 \\
    \vdots & \vdots & \ddots & \vdots & \vdots \\
    2 \alpha_{d, 1} & 2 \alpha_{d, 2} & \hdots & 2\alpha_{d, d} & \beta_d \\
    0 & 0 & \hdots & 0 & 0
  \end{bmatrix}
  \right)
  }_{\text{denote this as } M^t}
  \begin{bmatrix}
    w^{t}_1 \\
    w^{t}_2 \\
    \vdots \\
    w^{t}_d \\
    1
  \end{bmatrix}
\]

Hence, for such a ``quadratic'' model, each training example $(\bfx^t, y^t)$ is equivalent to a specific linear\footnote{Strictly speaking, these transformations are actually affine.} transformation $M^t$. However, we know that circuit gates (e.g. NAND) are nonlinear! Since the composition of linear transformations is still linear, we cannot encode a general circuit as a series of training examples for OGD.

As an aside, this suggests a fast method for approximately computing the weights of OGD on such a quadratic model after $\tau$ iterations. Specifically, consider the situtation where we OGD is repeatedly fed a sequence of $T$ points $(\bfx^1, y^1), (\bfx^2, y^2), ..., (\bfx^T, y^T)$ over and over again (in the same order) with initial weights $\bfw^1$. We want to know $\bfw^\tau$, the resulting weights after $\tau - 1$ iterations of OGD; we can compute these weights with only $O(T + \log \tau)$ matrix multiplications.

First, we compute the product $M = M^T M^{T-1} \cdots M^1$, which can be done with $(T-1) = O(T)$ matrix multiplications. Next, let $\tau' = \lfloor (\tau-1) / T \rfloor$. We compute $M^{\tau'}$ using the standard exponentiating by squaring trick, which requires $2\log_2 \tau' = O(\log \tau)$ matrix multiplications. Finally, we can apply the remaining $(\tau - 1) - T \tau' < T$ matrices through $O(T)$ more matrix multiplications. We take the resulting matrix and multiply it with our original weight vector. As claimed, we computed the new weight vector in only $O(T + \log \tau)$ matrix multplications.

The slight issue with the above method is that if we want to compute the weight vector exactly, the repeated squaring will rapidly increase the magnitude of the matrix entries and make multiplication expensive. It is possible to circumvent this issue by working with limited precision or over a finite field.
\section{API Implementation (Continued)}
\label{sec:appendix}

In this appendix, we implement the remaining functions of our API for soft-margin SVMs, which were listed in \hyperref[tab:public-api]{Table~\ref{tab:public-api}}.

\subsection{Implementation of \cpy{i_1}{i_2}}
Suppose we want to copy the $i_1$-th coordinate of the weight vector to its $i_2$-th coordinate. How can we do that using only gradient updates? The plan is to have a training example with both $x_{i_1}$ and $x_{i_2}$ nonzero. Intuitively, this first training example will ``read'' from $w_{i_1}$ and ``write'' to $w_{i_2}$ (it actually writes to both). We then perform some tidying so that the two possible states for each weight coordinate become $-1$ and $+1$. The sequence of operations together with the resulting weight vector after the gradient updates are provided in \hyperref[tab:cpy]{Table~\ref{tab:cpy}}. Observe that in the end, the value of the $i_2$-th coordinate of the weight vector is exactly the same as the $i_1$-coordinate and the operation \cpy{i_1}{i_2} is performed correctly.

The aforementioned read-write training example has label $+1$, $x_{i_1}=-4,x_{i_2}=+2$ and $x_i=0, \forall i\neq i_1,i_2$. After this example, we use a \fnot{i_1} gadget and the add trick to clean up.

\begin{itemize}
\item Let's focus in the case where $w_{i_1}=-1$ (upper half of every row in \hyperref[tab:cpy]{Table~\ref{tab:cpy}}). Without loss of generality let $w_{i_2}=0$ since otherwise we can just perform \reset{i_2} using previously defined gadgets. 

The gradient update on the first example will not affect the weight vector as $y \bfw\cdot \bfx = (+1)(-1)(-4)= 4> 1$. Then we just add $+2$ to get $(w_{i_1},w_{i_2})=(+1,0)$. After the not and the add trick, we end up with the desired $(w_{i_1},w_{i_2})=(-1,-1)$ outcome. 

\item This is similar to the previous case and by tracking down the gradient updates we end up with the desired $(w_{i_1},w_{i_2})=(+1,+1)$ outcome.
\end{itemize}

\begin{table}
\caption{Training data for \cpy{i_1}{i_2}.}
\label{tab:cpy}
\centering
\begin{tabular}{@{}cccr@{}r@{, }r@{$) \to ($}r@{, }r@{$)$}@{}}
  \toprule
  $x_{i_1}$ & $x_{i_2}$ & $y$ & \multicolumn{5}{c}{Effect on $(w_{i_1}, w_{i_2})$} \\ \midrule
  $-4$ & $2$ & $1$ & $($ & $-1$ & $0$ & $-1$ & $0$ \\
  \multicolumn{3}{c}{} & $($ & $1$ & $0$ & $-3$ & $2$ \\ \midrule
  $2$ & $0$ & $1$ & $($ & $-1$ & $0$ & $1$ & $0$ \\
  \multicolumn{3}{c}{(add trick)} & $($ & $-3$ & $2$ & $-1$ & $2$ \\ \midrule
  \multicolumn{3}{c}{\fnot{i_1}} & $($ & $1$ & $0$ & $-1$ & $0$ \\
  \multicolumn{3}{c}{} & $($ & $-1$ & $2$ & $1$ & $2$ \\ \midrule
  $0$ & $-1$ & $1$ & $($ & $-1$ & $0$ & $-1$ & $-1$ \\
  \multicolumn{3}{c}{(add trick)} & $($ & $1$ & $2$ & $1$ & $1$ \\
  \bottomrule
\end{tabular}
\end{table}

\subsection{Implementation of \dnand{i_1}{i_2}{i_3}}
We want to implement a NAND gate with inputs the coordinates $w_{i_1},w_{i_2}$ and output the result in $w_{i_3}$. Following our intuition, we will need a training example that is nonzero in $x_{i_1}, x_{i_2}$, and $x_{i_3}$, so that it can read the first two and write to the third. However, as before, such a training example necessarily modifies all three weights. To keep things simple, we will only ask our gadget to zero out $w_{i_1}$ and $w_{i_2}$, not restore them to their original values. This loss of input values is why we refer to this gadget as \emph{destructive} NAND. The operations needed are provided in \hyperref[tab:dnand]{Table~\ref{tab:dnand}}, and we only give the intuition regarding how this gadget was constructed.

As stated, our main training example will have nonzero values in all three coordinates. We would like to set things up so that the hinge criterion is satisfied only in the false case of NAND. To do so, we begin with an add trick which adds $-1$ to the third weight coordinate. Now, the sum of the three weights is either $-3$, $-1$, or $+1$, and this last case is the one we want to single out. For our main training example, we choose a magnitude of $2$ for our training values so that the possible sums become $-6$, $-2$, and $+2$; this puts the hinge threshold of $+1$ firmly between the two cases we care about. We finish with two reset gadgets and an add trick.

\begin{table}
\caption{Training data for \dnand{i_1}{i_2}{i_3}.}
\label{tab:dnand}
\centering
\begin{tabular}{@{}ccccr@{}r@{, }r@{, }r@{$) \to ($}r@{, }r@{, }r@{$)$}@{}}
  \toprule
  $x_{i_1}$ & $x_{i_2}$ & $x_{i_3}$ & $y$ & \multicolumn{7}{c}{Effect on $(w_{i_1}, w_{i_2}, w_{i_3})$} \\ \midrule
  $0$ & $0$ & $-1$ & $1$ & $($ & $-1$ & $-1$ & $0$ & $-1$ & $-1$ & $-1$ \\
  \multicolumn{4}{c}{(add trick)} & $($ & $-1$ & $1$ & $0$ & $-1$ & $1$ & $-1$ \\
  \multicolumn{4}{c}{} & $($ & $1$ & $-1$ & $0$ & $1$ & $-1$ & $-1$ \\
  \multicolumn{4}{c}{} & $($ & $1$ & $1$ & $0$ & $1$ & $1$ & $-1$ \\ \midrule
  $-2$ & $-2$ & $-2$ & $1$ & $($ & $-1$ & $-1$ & $-1$ & $-1$ & $-1$ & $-1$ \\
  \multicolumn{4}{c}{} & $($ & $-1$ & $1$ & $-1$ & $-1$ & $1$ & $-1$ \\
  \multicolumn{4}{c}{} & $($ & $1$ & $-1$ & $-1$ & $1$ & $-1$ & $-1$ \\
  \multicolumn{4}{c}{} & $($ & $1$ & $1$ & $-1$ & $-1$ & $-1$ & $-3$ \\ \midrule
  \multicolumn{4}{c}{\reset{i_1}} & $($ & $-1$ & $-1$ & $-1$ & $0$ & $-1$ & $-1$ \\
  \multicolumn{4}{c}{} & $($ & $-1$ & $1$ & $-1$ & $0$ & $1$ & $-1$ \\
  \multicolumn{4}{c}{} & $($ & $1$ & $-1$ & $-1$ & $0$ & $-1$ & $-1$ \\
  \multicolumn{4}{c}{} & $($ & $-1$ & $-1$ & $-3$ & $0$ & $-1$ & $-3$ \\ \midrule
  \multicolumn{4}{c}{\reset{i_2}} & $($ & $0$ & $-1$ & $-1$ & $0$ & $0$ & $-1$ \\
  \multicolumn{4}{c}{} & $($ & $0$ & $1$ & $-1$ & $0$ & $0$ & $-1$ \\
  \multicolumn{4}{c}{} & $($ & $0$ & $-1$ & $-1$ & $0$ & $0$ & $-1$ \\
  \multicolumn{4}{c}{} & $($ & $0$ & $-1$ & $-3$ & $0$ & $0$ & $-3$ \\ \midrule
  $0$ & $0$ & $2$ & $1$ & $($ & $0$ & $0$ & $-1$ & $0$ & $0$ & $1$ \\
  \multicolumn{4}{c}{(add trick)} & $($ & $0$ & $0$ & $-1$ & $0$ & $0$ & $1$ \\
  \multicolumn{4}{c}{} & $($ & $0$ & $0$ & $-1$ & $0$ & $0$ & $1$ \\
  \multicolumn{4}{c}{} & $($ & $0$ & $0$ & $-3$ & $0$ & $0$ & $-1$ \\
  \bottomrule
\end{tabular}
\end{table}

\subsection{Implementation of \inputF{i_1}}
The effect of \inputF{i_1} is to map the $i_1$-th coordinate (which is either $-1,0,+1$) to $-1$, unless it is $+1$ in which case it should remain $+1$. The 4 steps in \hyperref[tab:inputF]{Table~\ref{tab:inputF}} with the \texttt{add} gadgets should be clear by now. Here we give the calculations of the gradients and updates for the 3 steps that contain training examples.

\begin{itemize}
\item The training example has label $y=+1$, with $x_{i_1}=+3$ and $x_{i}=0,\forall i\neq i_1$. If $w_{i_1}=0$ then $y \bfw\cdot \bfx=(+1)(0)=0 < 1$ so the gradient step will add $yx_{i_1}=(+1)(+3)=3$ to $w_{i_1}$. If $w_{i_1}=+1$ then $y \bfw\cdot \bfx=(+1)(+1)(+3)=3 > 1$ so there is no update. If $w_{i_1}=+2$, then again there is no update. 

\item The training example has label $y=+1$, with $x_{i_1}=+2$ and $x_{i}=0,\forall i\neq i_1$. If $w_{i_1}=+2$ then $y \bfw\cdot \bfx=(+1)(+2)(+2)=+4 > 1$ so there is no update. If $w_{i_1}=0$, then $y \bfw\cdot \bfx=0 < 1$, so the gradient step will add $yx_{i_1}=(+1)(+2)=2$ to $w_{i_1}$. If $w_{i_1}=+1$ then $y \bfw\cdot \bfx=(+1)(+1)(+2)=2 > 1$ so there is no update. 

\item Training on the final training example is similar to the first case above.
\end{itemize}

\begin{table}
\caption{Training data for \inputF{i_1}.}
\label{tab:inputF}
\centering
\begin{tabular}{@{}ccr@{}r@{$) \to ($}r@{$)$}@{}}
  \toprule
  $x_{i_1}$ & $y$ & \multicolumn{3}{c}{Effect on $(w_{i_1})$} \\ \midrule
  $-\tfrac14$ & $1$ & $($ & $-1$ & $-\tfrac54$ \\
  \multicolumn{2}{c}{} & $($ & $0$ & $-\tfrac14$ \\
  \multicolumn{2}{c}{} & $($ & $1$ & $\tfrac34$ \\ \midrule
  $-1$ & $1$ & $($ & $-\tfrac54$ & $-\tfrac54$ \\
  \multicolumn{2}{c}{} & $($ & $-\tfrac14$ & $-\tfrac54$ \\
  \multicolumn{2}{c}{} & $($ & $\tfrac34$ & $-\tfrac14$ \\ \midrule
  $-3$ & $1$ & $($ & $-\tfrac54$ & $-\tfrac54$ \\
  \multicolumn{2}{c}{} & $($ & $-\tfrac54$ & $-\tfrac54$ \\
  \multicolumn{2}{c}{} & $($ & $-\tfrac14$ & $-\tfrac{13}{4}$ \\ \midrule
  $\tfrac94$ & $1$ & $($ & $-\tfrac54$ & $1$ \\
  \multicolumn{2}{c}{(add trick)} & $($ & $-\tfrac54$ & $1$ \\
  \multicolumn{2}{c}{} & $($ & $-\tfrac{13}{4}$ & $-1$ \\ \midrule
  \multicolumn{2}{c}{\fnot{i_1}} & $($ & $1$ & $-1$ \\
  \multicolumn{2}{c}{} & $($ & $1$ & $-1$ \\
  \multicolumn{2}{c}{} & $($ & $-1$ & $1$ \\
  \bottomrule
\end{tabular}
\end{table}

\subsection{Implementation of \setiftrue{i_1}{i_2}}
This short gadget is given two coordinates $i_1,i_2$ and sets $w_{i_2}=+1$ only if $w_{i_1}=+1$, otherwise everything stays unchanged. We use it to decide if at any point in the circuit computation, the target binary string $s^*$ is ever reached, in which case a specially reserved bit in the weight vector (e.g. the first bit of the $w$) is set to 1 to signal this fact.

We are going to use one training example, an add trick and then a {\texttt{not}} gadget and the calculations explaining the derivations of \hyperref[tab:setiftrue]{Table~\ref{tab:setiftrue}} are given below:

\begin{itemize}
\item The first training example has label $y=+1$, with $x_{i_1}=-4,x_{i_2}=+1$ and $x_{i}=0,\forall i\neq i_1,i_2$. If $w_{i_1}=-1,w_{i_2}=0$ then $y \bfw\cdot \bfx=(+1)(+4)=+4 > 1$ so there is no update. If $w_{i_1}=+1,w_{i_2}=0$ then $y \bfw\cdot \bfx=(+1)(+1)(-4)=-4 < 1$, so the gradient step will add $yx_{i_1}=(+1)(-4)=-4$ to $w_{i_1}$ (which now becomes $-3$) and $yx_{i_2}=(+1)(+1)=+1$ to $w_{i_2}$ (which now becomes $+1$). 

\item Then, we perform the add trick mentioned above with the training example that has label $y=+1$, with $x_{i_1}=2,x_{i_2}=0$ and $x_{i}=0,\forall i\neq i_1,i_2$ and finally we use a {\texttt{not}} gadget. The corresponding weight updates are shown in \hyperref[tab:setiftrue]{Table~\ref{tab:setiftrue}}.

%If $w_{i_1}=+1, w_{i_2}=0$ then $y \bfw\cdot \bfx=(+1)(+1)(-4)=-4 < 1$, so the gradient step will add $yx_{i_1}=(+1)(-4)=-4$ to $w_{i_1}$ and $yx_{i_2}=0$ to $w_{i_2}$ as is shown in \hyperref[tab:setiftrue]{Table~\ref{tab:setiftrue}}. If $w_{i_1}=-1,w_{i_2}=+1$ then $y \bfw\cdot \bfx=(+1)(-1)(-4)=+4 > 1$ and so there would be no update.
\end{itemize}

\begin{table}
\caption{Training data for \setiftrue{i_1}{i_2}.}
\label{tab:setiftrue}
\centering
\begin{tabular}{@{}cccr@{}r@{, }r@{$) \to ($}r@{, }r@{$)$}@{}}
  \toprule
  $x_{i_1}$ & $x_{i_2}$ & $y$ & \multicolumn{5}{c}{Effect on $(w_{i_1}, w_{i_2})$} \\ \midrule
  $-4$ & $1$ & $1$ & $($ & $-1$ & $0$ & $-1$ & $0$ \\
  \multicolumn{3}{c}{} & $($ & $1$ & $0$ & $-3$ & $1$ \\ \midrule
  $2$ & $0$ & $1$ & $($ & $-1$ & $0$ & $1$ & $0$ \\
  \multicolumn{3}{c}{(add trick)} & $($ & $-3$ & $1$ & $-1$ & $1$ \\ \midrule
  \multicolumn{3}{c}{\fnot{i_1}} & $($ & $1$ & $0$ & $-1$ & $0$ \\
  \multicolumn{3}{c}{} & $($ & $-1$ & $1$ & $1$ & $1$ \\
  \bottomrule
\end{tabular}
\end{table}

\section{Proof Extension for Regularization (Continued)}
\label{sec:regularizer}

In this appendix, we give an augmented API for regularization, show how to modify the original reduction to use the augmented API, and then give an implementation of the API.

\subsection{Augmented API for Regularization}

Our augmented API is listed in \hyperref[tab:regularizer-api]{Table~\ref{tab:regularizer-api}}. These five functions serve the same purpose as the functions of our original API (see \hyperref[tab:public-api]{Table~\ref{tab:public-api}}), but now accept additional parameters and have return values so that our reduction can keep track of the magnitude of each weight.

All gadgets here,
\rreset{i_1}{\epsilon_1},
\rdnand{i_1}{i_2}{i_3}{\epsilon_1}{\epsilon_2},
\rinputF{i_1}{\epsilon_1},
and \rsetiftrue{i_1}{i_2}{\epsilon_1} have essentially the same behavior as before, but now accept magnitude parameters and output the final magnitude of the weights that they write to. A more drastic change was made to \rcopy{i_1}{i_2}{i_3}{\epsilon_1}, which now destroys the bit stored in its input weight. To compensate, it now makes two copies, so that using it increases the total number of copies of a weight.

\subsection{Reduction Modifications for Regularization}

\newcommand{\specw}{\triangle}

Our reduction still performs the same transformation of $\C$ into $\C'$. However, we will use an additional dimension (now $d = n + m + 4$), which we also denote with a new special: $\specw$. As stated before, we keep a counter $\epsilon_i$ for each dimension $i$, decaying all counters by $\alpha$ after each training example we produce.

In most cases, the appropriate $\epsilon_i$ to pass to our gadgets is clear: we take the last $\epsilon_i$ we received from a gadget writing to this coordinate and decay it appropriately. There is one major exception: in the first phase of the reduction, we need to iterate over $i = 1, 2, \ldots, n$ and call \rinputF{i}{\epsilon_i}. The correct input magnitude is actually based on the last time these weights were possibly edited, which is actually in the (previous pass over the data) fourth phase of the reduction! Luckily, in our implementation of this API the number of training examples to implement a gadget \emph{does not depend} on the inputs $\epsilon_i$. As a result, we can either pick the appropriate values knowing the contents of all the phases, or we can run the reduction once with $\epsilon_i = 1$ and then perform a second pass once we know the total number of training examples and which training examples are associated with which API calls. One important consequence of this reasoning is that since the reduction touches each coordinate at least once as we pass over all training examples, the maximum decay of any weight is only singly-exponential in the number of training examples (which is polynomial in the original circuit problem size), which is better than the naive bound of double-exponential. As a result, we only require polynomial bits of precision are needed to represent the weights at any point in time. Note that if one does not care about regularization, then all of our other constructions only required \emph{fixed precision}.

Other than managing these magnitudes, we also alter the second and fourth phase of our reduction to account for a revised copy function (this is why we need an additional dimension). In the new second phase of our reduction, we iterate over $i = n+1, n+2, \ldots, n+m$. Again, we look at the associated NAND gate with inputs $i_1, i_2$. We call:
\begin{itemize}
  \item \rcopy{i_1}{\specy}{\specw}{\cdot},
  \item \rreset{\specy}{\cdot},
  \item \rcopy{\specw}{i_1}{\specy}{\cdot},
  \item \rcopy{i_2}{\specz}{\specw}{\cdot},
  \item \rreset{\specz}{\cdot},
  \item \rcopy{\specw}{i_2}{\specz}{\cdot}, and
  \item \rdnand{\specy}{\specz}{i}{\cdot}{\cdot},
\end{itemize}
in that order with appropriate $\epsilon_i$.

Similarly, in the fourth phase of our reduction, we iterate over $i = 1, 2, \ldots, n$ and call \rreset{i}{\cdot}, \rcopy{i_1}{i}{\specy}{\cdot}, \rcopy{\specy}{i_1}{\specz}{\cdot}, \rreset{\specz}{\cdot}, in that order with appropriate $\epsilon_i$.

The reason the reduction works is the same as before: the reduction forces the weights to simulate computation of the circuit and a check for $s^*$ with each pass through the training data. This completes the description of how to modify the reduction.

\subsection{Implementation of \rreset{i_1}{\epsilon_1}}

At a high level, the idea behind this implementation is as follows. We are given a weight that either contains a small negative or a small positive value. We would like to add the difference between these two potential values, but only in the case where the original value is negative. In order to do so, we must first increase both possible values so that when multiplied by their original difference, one falls below and one falls above our comparison threshold of $+1$.

\begin{table}
\caption{Training data for \rreset{i_1}{\epsilon_1}.}
\label{tab:rreset}
\centering
\begin{tabular}{@{}ccr@{}r@{$) \to ($}r@{$)$}@{}}
  \toprule
  $x_{i_1}$ & $y$ & \multicolumn{3}{c}{Effect on $(w_{i_1})$} \\ \midrule
  $\frac{1}{2\epsilon_1 \alpha^2}$ & $1$ & $($ & $-\epsilon_1$ & $\frac{1}{2\epsilon_1 \alpha^2} - \epsilon_1 \alpha$ \\
  \multicolumn{2}{c}{} & $($ & $\epsilon_1$ & $\frac{1}{2\epsilon_1 \alpha^2} + \epsilon_1 \alpha$ \\ \midrule
  $2 \epsilon_1 \alpha^2$ & $1$ & $($ & $\frac{1}{2\epsilon_1 \alpha^2} - \epsilon_1 \alpha$ & $\frac{1}{2\epsilon_1 \alpha} + \epsilon_1 \alpha^2$ \\
  \multicolumn{2}{c}{} & $($ & $\frac{1}{2\epsilon_1 \alpha^2} + \epsilon_1 \alpha$ & $\frac{1}{2\epsilon_1 \alpha} + \epsilon_1 \alpha^2$ \\ \midrule
  $-\frac{1}{2\epsilon_1} - \epsilon_1 \alpha^3$ & $1$ & $($ & $\frac{1}{2\epsilon_1 \alpha} + \epsilon_1 \alpha^2$ & $0$ \\
  \multicolumn{2}{c}{} & $($ & $\frac{1}{2\epsilon_1 \alpha} + \epsilon_1 \alpha^2$ & $0$ \\
  \bottomrule
\end{tabular}
\end{table}

The training data that executes this plan is given in \hyperref[tab:reset]{Table~\ref{tab:rreset}}. The first training example has a small magnitude so that both possibilities receive a gradient update:
\[
  \frac{1}{2\epsilon_1 \alpha^2} \cdot \epsilon_1 = \frac{1}{2\alpha^2}.
\]
Note that the RHS is at most $1$ due to the range of $\alpha$. This update sets up for the second training example. Observe that:
\[
  2 \epsilon_1 \alpha^2 \cdot \frac{1}{2\epsilon_1 \alpha^2} = 1
\]
so that the loss or gain of $\epsilon_1 \alpha$ pushes our first possibility below the threshold and our second possibility above the threshold of $+1$. We have now collapsed our two possibilities into only a single possibility. The third training example triggers an update because $x$ and $w$ have a negative dot product, and the term is chosen to cancel out the remaining value.

\begin{table}
\caption{Augmented API for Regularization. $\sigma(w_i)$ denotes the sign function.}
\label{tab:regularizer-api}
\centering
\begin{tabular}{@{}llll@{}}
  \toprule
  \textbf{Function} & \textbf{Precondition(s)} & \textbf{Returns} & \textbf{Description} \\ \midrule
  % reset(i_1, eps_1)
  \rreset{i_1}{\epsilon_1} &
  $i_1 \in \memory$ &
  None &
  $w_{i_1} \leftarrow 0$ \\
  (for implementation, see  \hyperref[tab:reset]{Table~\ref{tab:rreset}})&
  $w_{i_1} \in \rbinary{1}$ &
  &
  \\ \midrule
  % copy2(i_1, i_2, i_3, eps_1)
  \rcopy{i_1}{i_2}{i_3}{\epsilon_1} &
  $i_1, i_2, i_3 \in \memory$ &
  $(\epsilon_2, \epsilon_3)$ &
  $w_{i_2} \leftarrow \sigma(w_{i_1}) \epsilon_2$ \\
  (for implementation, see  \hyperref[tab:rcopy]{Table~\ref{tab:rcopy}})&
  $w_{i_1} \in \rbinary{1}$ &
  &
  $w_{i_3} \leftarrow \sigma(w_{i_1}) \epsilon_3$ \\
  &
  $w_{i_2} = 0$ &
  &
  \\
  &
  $w_{i_3} = 0$ &
  &
  \\ \midrule
  % destructive_nand(i_1, i_2, i_3, eps_1, eps_2)
  \rdnand{i_1}{i_2}{i_3}{\epsilon_1}{\epsilon_2} &
  $i_1, i_2, i_3 \in \memory$ &
  $(\epsilon_3)$ &
  $w_{i_3} \leftarrow \nand{\sigma(w_{i_1})}{\sigma(w_{i_2})} \epsilon_3$ \\
  (for implementation, see  \hyperref[tab:rdnand]{Table~\ref{tab:rdnand}})&
  $w_{i_1} \in \rbinary{1}$ &
  &
  $w_{i_1} \leftarrow 0$ \\
  &
  $w_{i_2} \in \rbinary{2}$ &
  &
  $w_{i_2} \leftarrow 0$ \\ \midrule
  % input_false(i_1, eps_1)
  \rinputF{i_1}{\epsilon_1} &
  $i_1 \in \memory$ &
  $(\epsilon_1')$ &
  If $w_{i_1} = 0$, $w_{i_1} \leftarrow -\epsilon_1'$ \\
  (for implementation, see  \hyperref[tab:rinputF]{Table~\ref{tab:rinputF}})&
  $w_{i_1} \in \rubinary{1}$ &
  &
  Else, $w_{i_1} \leftarrow \sigma(w_{i_1}) \epsilon_1'$ \\ \midrule
  % set_if_true(i_1, i_2, eps_1)
  \rsetiftrue{i_1}{i_2}{\epsilon_1} &
  $i_1, i_2 \in \memory$ &
  $(\epsilon_1', \epsilon_2)$ &
  If $w_{i_1} > 0$, $w_{i_2} \leftarrow +\epsilon_2$ \\
  (for implementation, see  \hyperref[tab:rsetiftrue]{Table~\ref{tab:rsetiftrue}})&
  $w_{i_1} \in \rbinary{1}$ &
  &
  If $w_{i_1} < 0$, $w_{i_2}$ remains at $0$ \\
  &
  $w_{i_2} = 0$ &
  &
  (including in intermediate steps) \\
  &
  &
  &
  $w_{i_1} \leftarrow \sigma(w_{i_1}) \epsilon_1'$ \\
  \bottomrule
\end{tabular}
\end{table}

\subsection{Implementation of \rcopy{i_1}{i_2}{i_3}{\epsilon_1}}

At a high level, the idea behind this implementation is as follows. We are given a weight that either contains a small negative or a small positive value. Using a large multiplier, we can detect the sign of this weight and copy the sign into two other weights. We then cleanup and make the original weight zero.

\begin{table}
\caption{Training data for \rcopy{i_1}{i_2}{i_3}{\epsilon_1}.}
\label{tab:rcopy}
\centering
\begin{tabular}{@{}ccccr@{}r@{, }r@{, }r@{$) \to ($}r@{, }r@{, }r@{$)$}@{}}
  \toprule
  $x_{i_1}$ & $x_{i_2}$ & $x_{i_3}$ & $y$ & \multicolumn{7}{c}{Effect on $(w_{i_1}, w_{i_2}, w_{i_3})$} \\ \midrule
  $\frac{2}{\epsilon_1}$ & $-2$ & $-2$ & $1$ & $($ & $-\epsilon_1$ & $0$ & $0$ & $\frac{2}{\epsilon_1}-\epsilon_1 \alpha$ & $-2$ & $-2$ \\
  \multicolumn{4}{c}{} & $($ & $\epsilon_1$ & $0$ & $0$ & $\epsilon_1 \alpha$ & $0$ & $0$ \\ \midrule
  $-\frac{\alpha}{\epsilon_1}$ & $\alpha$ & $\alpha$ & $1$ & $($ & $\frac{2}{\epsilon_1}-\epsilon_1 \alpha$ & $-2$ & $-2$ & $\frac{\alpha}{\epsilon_1} - \epsilon_1 \alpha^2$ & $-\alpha$ & $-\alpha$ \\
  \multicolumn{4}{c}{} & $($ & $\epsilon_1 \alpha$ & $0$ & $0$ & $-\frac{\alpha}{\epsilon_1} + \epsilon_1 \alpha^2$ & $\alpha$ & $\alpha$ \\ \midrule
  \multicolumn{4}{c}{\rreset{i_1}{\frac{\alpha}{\epsilon_1} - \epsilon_1 \alpha^2}} & $($ & $\frac{\alpha}{\epsilon_1} - \epsilon_1 \alpha^2$ & $-\alpha$ & $-\alpha$ & $0$ & $-\alpha^4$ & $-\alpha^4$ \\
  \multicolumn{4}{c}{} & $($ & $-\frac{\alpha}{\epsilon_1} + \epsilon_1 \alpha^2$ & $\alpha$ & $\alpha$ & $0$ & $\alpha^4$ & $\alpha^4$ \\ \midrule
  \multicolumn{11}{c}{Return $(\epsilon_2 = \alpha^4, \epsilon_3 = \alpha^4)$.} \\
  \bottomrule
\end{tabular}
\end{table}

The training data that executes this plan is given in \hyperref[tab:rcopy]{Table~\ref{tab:rcopy}}. The first training example has enough magnitude so that the resulting product has magnitude $2$:
\[
  \frac{2}{\epsilon_1} \cdot \epsilon_1 = 2
\]
In the second update, we recenter around zero. In particular, we observe that $+\frac{2}{\epsilon_1}-\epsilon_1 \alpha$ is positive, so every component of $(w \cdot x)$ in this step is in fact negative, triggering an update.

We finish by using our reset gadget to clean up $w_{i_1}$, noting that it uses three training examples and our other weights continue to decay in the meantime.

\subsection{Implementation of \rdnand{i_1}{i_2}{i_3}{\epsilon_1}{\epsilon_2}}

At a high level, the idea behind this implementation is as follows. The idea is similar to our original NAND gate, where we used the observation that if two weights are $\pm 1$, we can use a threshold on their sum to compute NAND: when the sum is $-2$ or $0$, the result is true, and when the sum is $+2$, the result is false. We use this sum to put the result of the NAND computation into the third weight. Unfortunately, this results in the first two weights being in one of three possible states each, and some work is needed to clean them up as well. Finally, the third state should be made into the form $\pm \epsilon_3$.

% TODO(jrwang): Fix this table and also the other d_nand table, observing that (-2 -2 +2 +1) is okay for first training example.
\begin{table}
\small
\caption{Training data for \rdnand{i_1}{i_2}{i_3}{\epsilon_1}{\epsilon_2}.}
\label{tab:rdnand}
\centering
\begin{tabular}{@{}ccccr@{}r@{, }r@{, }r@{$) \to ($}r@{, }r@{, }r@{$)$}@{}}
  \toprule
  $x_{i_1}$ & $x_{i_2}$ & $x_{i_3}$ & $y$ & \multicolumn{7}{c}{Effect on $(w_{i_1}, w_{i_2}, w_{i_3})$} \\ \midrule
  $0$ & $0$ & $-1$ & $1$ & $($ & $-\epsilon_1$ & $-\epsilon_2$ & $0$ & $-\epsilon_1 \alpha$ & $-\epsilon_2 \alpha$ & $-1$ \\
  \multicolumn{4}{c}{} & $($ & $-\epsilon_1$ & $\epsilon_2$ & $0$ & $-\epsilon_1 \alpha$ & $\epsilon_2 \alpha$ & $-1$ \\
  \multicolumn{4}{c}{} & $($ & $\epsilon_1$ & $-\epsilon_2$ & $0$ & $\epsilon_1 \alpha$ & $-\epsilon_2 \alpha$ & $-1$ \\
  \multicolumn{4}{c}{} & $($ & $\epsilon_1$ & $\epsilon_2$ & $0$ & $\epsilon_1 \alpha$ & $\epsilon_2 \alpha$ & $-1$ \\ \midrule
  $-\frac{4}{\epsilon_1 \alpha}$ & $-\frac{4}{\epsilon_2 \alpha}$ & $-2 \alpha$ & $1$ & $($ & $-\epsilon_1 \alpha$ & $-\epsilon_2 \alpha$ & $-1$ & $-\epsilon_1 \alpha^2$ & $-\epsilon_2 \alpha^2$ & $-\alpha$ \\
  \multicolumn{4}{c}{} & $($ & $-\epsilon_1 \alpha$ & $\epsilon_2 \alpha$ & $-1$ & $-\epsilon_1 \alpha^2$ & $\epsilon_2 \alpha^2$ & $-\alpha$ \\
  \multicolumn{4}{c}{} & $($ & $\epsilon_1 \alpha$ & $-\epsilon_2 \alpha$ & $-1$ & $\epsilon_1 \alpha^2$ & $-\epsilon_2 \alpha^2$ & $-\alpha$ \\
  \multicolumn{4}{c}{} & $($ & $\epsilon_1 \alpha$ & $\epsilon_2 \alpha$ & $-1$ & $-\frac{4}{\epsilon_1 \alpha}+\epsilon_1 \alpha^2$ & $-\frac{4}{\epsilon_2 \alpha}+\epsilon_2 \alpha^2$ & $-3 \alpha$ \\ \midrule
  $\frac{4}{\epsilon_1}$ & $0$ & $0$ & $1$ & $($ & $-\epsilon_1 \alpha^2$ & $-\epsilon_2 \alpha^2$ & $-\alpha$ & $\frac{4}{\epsilon_1}-\epsilon_1 \alpha^3$ & $-\epsilon_2 \alpha^3$ & $-\alpha^2$ \\
  \multicolumn{4}{c}{} & $($ & $-\epsilon_1 \alpha^2$ & $\epsilon_2 \alpha^2$ & $-\alpha$ & $\frac{4}{\epsilon_1}-\epsilon_1 \alpha^3$ & $\epsilon_2 \alpha^3$ & $-\alpha^2$ \\
  \multicolumn{4}{c}{} & $($ & $\epsilon_1 \alpha^2$ & $-\epsilon_2 \alpha^2$ & $-\alpha$ & $\epsilon_1 \alpha^3$ & $-\epsilon_2 \alpha^3$ & $-\alpha^2$ \\
  \multicolumn{4}{c}{} & $($ & $-\frac{4}{\epsilon_1 \alpha}+\epsilon_1 \alpha^2$ & $-\frac{4}{\epsilon_2 \alpha}+\epsilon_2 \alpha^2$ & $-3 \alpha$ & $\epsilon_1 \alpha^3$ & $-\frac{4}{\epsilon_2}+\epsilon_2 \alpha^3$ & $-3\alpha^2$ \\ \midrule
  $0$ & $\frac{4\alpha}{\epsilon_2}$ & $0$ & $1$ & $($ & $\frac{4}{\epsilon_1}-\epsilon_1 \alpha^3$ & $-\epsilon_2 \alpha^3$ & $-\alpha^2$ & $\frac{4\alpha}{\epsilon_1}-\epsilon_1 \alpha^4$ & $\frac{4\alpha}{\epsilon_2}-\epsilon_2 \alpha^4$ & $-\alpha^3$ \\
  \multicolumn{4}{c}{} & $($ & $\frac{4}{\epsilon_1}-\epsilon_1 \alpha^3$ & $\epsilon_2 \alpha^3$ & $-\alpha^2$ & $\frac{4\alpha}{\epsilon_1}-\epsilon_1 \alpha^4$ & $\epsilon_2 \alpha^4$ & $-\alpha^3$ \\
  \multicolumn{4}{c}{} & $($ & $\epsilon_1 \alpha^3$ & $-\epsilon_2 \alpha^3$ & $-\alpha^2$ & $\epsilon_1 \alpha^4$ & $\frac{4\alpha}{\epsilon_2}-\epsilon_2 \alpha^4$ & $-\alpha^3$ \\
  \multicolumn{4}{c}{} & $($ & $\epsilon_1 \alpha^3$ & $-\frac{4}{\epsilon_2}+\epsilon_2 \alpha^3$ & $-3\alpha^2$ & $\epsilon_1 \alpha^4$ & $\epsilon_2 \alpha^4$ & $-3\alpha^3$ \\ \midrule
  $-\frac{2\alpha^2}{\epsilon_1}$ & $0$ & $0$ & $1$ & $($ & $\frac{4\alpha}{\epsilon_1}-\epsilon_1 \alpha^4$ & $\frac{4\alpha}{\epsilon_2}-\epsilon_2 \alpha^4$ & $-\alpha^3$ & $\frac{2\alpha^2}{\epsilon_1}-\epsilon_1 \alpha^5$ & $\frac{4\alpha^2}{\epsilon_2}-\epsilon_2 \alpha^5$ & $-\alpha^4$ \\
  \multicolumn{4}{c}{} & $($ & $\frac{4\alpha}{\epsilon_1}-\epsilon_1 \alpha^4$ & $\epsilon_2 \alpha^4$ & $-\alpha^3$ & $\frac{2\alpha^2}{\epsilon_1}-\epsilon_1 \alpha^5$ & $\epsilon_2 \alpha^5$ & $-\alpha^4$ \\
  \multicolumn{4}{c}{} & $($ & $\epsilon_1 \alpha^4$ & $\frac{4\alpha}{\epsilon_2}-\epsilon_2 \alpha^4$ & $-\alpha^3$ & $-\frac{2\alpha^2}{\epsilon_1}+\epsilon_1 \alpha^5$ & $\frac{4\alpha^2}{\epsilon_2}-\epsilon_2 \alpha^5$ & $-\alpha^4$ \\
  \multicolumn{4}{c}{} & $($ & $\epsilon_1 \alpha^4$ & $\epsilon_2 \alpha^4$ & $-3\alpha^3$ & $-\frac{2\alpha^2}{\epsilon_1}+\epsilon_1 \alpha^5$ & $\epsilon_2 \alpha^5$ & $-3\alpha^4$ \\ \midrule
  \multicolumn{4}{c}{\rreset{i_1}{+\frac{2\alpha^2}{\epsilon_1}-\epsilon_1\alpha^5}} & $($ & $\frac{2\alpha^2}{\epsilon_1}-\epsilon_1 \alpha^5$ & $\frac{4\alpha^2}{\epsilon_2}-\epsilon_2 \alpha^5$ & $-\alpha^4$ & $0$ & $\frac{4\alpha^5}{\epsilon_2}-\epsilon_2 \alpha^8$ & $-\alpha^7$ \\
  \multicolumn{4}{c}{} & $($ & $\frac{2\alpha^2}{\epsilon_1}-\epsilon_1 \alpha^5$ & $\epsilon_2 \alpha^5$ & $-\alpha^4$ & $0$ & $\epsilon_2 \alpha^8$ & $-\alpha^7$ \\
  \multicolumn{4}{c}{} & $($ & $-\frac{2\alpha^2}{\epsilon_1}+\epsilon_1 \alpha^5$ & $\frac{4\alpha^2}{\epsilon_2}-\epsilon_2 \alpha^5$ & $-\alpha^4$ & $0$ & $\frac{4\alpha^5}{\epsilon_2}-\epsilon_2 \alpha^8$ & $-\alpha^7$ \\
  \multicolumn{4}{c}{} & $($ & $-\frac{2\alpha^2}{\epsilon_1}+\epsilon_1 \alpha^5$ & $\epsilon_2 \alpha^5$ & $-3\alpha^4$ & $0$ & $\epsilon_2 \alpha^8$ & $-3\alpha^7$ \\ \midrule
  $0$ & $-\frac{2\alpha^6}{\epsilon_1}$ & $0$ & $1$ & $($ & $0$ & $\frac{4\alpha^5}{\epsilon_2}-\epsilon_2 \alpha^8$ & $-\alpha^7$ & $0$ & $\frac{2\alpha^6}{\epsilon_2}-\epsilon_2 \alpha^9$ & $-\alpha^8$ \\
  \multicolumn{4}{c}{} & $($ & $0$ & $\epsilon_2 \alpha^8$ & $-\alpha^7$ & $0$ & $-\frac{2\alpha^6}{\epsilon_2}+\epsilon_2 \alpha^9$ & $-\alpha^8$ \\
  \multicolumn{4}{c}{} & $($ & $0$ & $\frac{4\alpha^5}{\epsilon_2}-\epsilon_2 \alpha^8$ & $-\alpha^7$ & $0$ & $\frac{2\alpha^6}{\epsilon_2}-\epsilon_2 \alpha^9$ & $-\alpha^8$ \\
  \multicolumn{4}{c}{} & $($ & $0$ & $\epsilon_2 \alpha^8$ & $-3\alpha^7$ & $0$ & $-\frac{2\alpha^6}{\epsilon_2}+\epsilon_2 \alpha^9$ & $-3\alpha^8$ \\ \midrule
  \multicolumn{4}{c}{\rreset{i_2}{+\frac{2\alpha^6}{\epsilon_2}-\epsilon_2\alpha^9}} & $($ & $0$ & $\frac{2\alpha^6}{\epsilon_2}-\epsilon_2 \alpha^9$ & $-\alpha^8$ & $0$ & $0$ & $-\alpha^{11}$ \\
  \multicolumn{4}{c}{} & $($ & $0$ & $-\frac{2\alpha^6}{\epsilon_2}+\epsilon_2 \alpha^9$ & $-\alpha^8$ & $0$ & $0$ & $-\alpha^{11}$ \\
  \multicolumn{4}{c}{} & $($ & $0$ & $\frac{2\alpha^6}{\epsilon_2}-\epsilon_2 \alpha^9$ & $-\alpha^8$ & $0$ & $0$ & $-\alpha^{11}$ \\
  \multicolumn{4}{c}{} & $($ & $0$ & $-\frac{2\alpha^6}{\epsilon_2}+\epsilon_2 \alpha^9$ & $-3\alpha^8$ & $0$ & $0$ & $-3\alpha^{11}$ \\ \midrule
  $0$ & $0$ & $2\alpha^{12}$ & $1$ & $($ & $0$ & $0$ & $-\alpha^{11}$ & $0$ & $0$ & $\alpha^{12}$ \\
  \multicolumn{4}{c}{} & $($ & $0$ & $0$ & $-\alpha^{11}$ & $0$ & $0$ & $\alpha^{12}$ \\
  \multicolumn{4}{c}{} & $($ & $0$ & $0$ & $-\alpha^{11}$ & $0$ & $0$ & $\alpha^{12}$ \\
  \multicolumn{4}{c}{} & $($ & $0$ & $0$ & $-3\alpha^{11}$ & $0$ & $0$ & $-\alpha^{12}$ \\ \midrule
  \multicolumn{11}{c}{Return $(\epsilon_3 = \alpha^{12})$.} \\
  \bottomrule
\end{tabular}
\end{table}

The training data that executes this plan is given in \hyperref[tab:rdnand]{Table~\ref{tab:rdnand}}. Note that the training examples with entries $(+\frac{4}{\epsilon_1}, 0, 0, +1)$ and $(0, +\frac{4\alpha}{\epsilon_2}, 0, +1)$ only have the listed effect due to our bounds on $\alpha$. In particular, one possible value of $(w \cdot x)$ is:
\[
  +\frac{4\alpha}{\epsilon_2} \cdot \epsilon_2 \alpha^3 = 4 \alpha^4
\]
which is only greater than $+1$ due to our bounds on $\alpha$.

\subsection{Implementation of \rinputF{i_1}{\epsilon_1}}

At a high level, the idea behind this implementation is as follows. We have three possible states. Our first training example only triggers on the nonnegative cases, while our second training example triggers on the negative case. The difference between these two updates is designed so that the negative case and zero case map to the same value. After that, we finish by performing a translation so that the cases fall into the form $\pm \epsilon_1'$.

\begin{table}
\caption{Training data for \rinputF{i_1}{\epsilon_1}.}
\label{tab:rinputF}
\centering
\begin{tabular}{@{}ccr@{}r@{$) \to ($}r@{$)$}@{}}
  \toprule
  $x_{i_1}$ & $y$ & \multicolumn{3}{c}{Effect on $(w_{i_1})$} \\ \midrule
  $\left(-\frac{1}{\epsilon_1}-\epsilon_1 \alpha\right)$ & $1$ & $($ & $-\epsilon_1$ & $-\epsilon_1 \alpha$ \\
  \multicolumn{2}{c}{} & $($ & $0$ & $-\frac{1}{\epsilon_1}-\epsilon_1 \alpha$ \\
  \multicolumn{2}{c}{} & $($ & $\epsilon_1$ & $-\frac{1}{\epsilon_1}\hphantom{{}-\epsilon_1 \alpha}$ \\ \midrule
  $-\frac{\alpha}{\epsilon_1}$ & $1$ & $($ & $-\epsilon_1 \alpha$ & $-\frac{\alpha}{\epsilon_1}-\epsilon_1 \alpha^2$ \\
  \multicolumn{2}{c}{} & $($ & $-\frac{1}{\epsilon_1}-\epsilon_1 \alpha$ & $-\frac{\alpha}{\epsilon_1}-\epsilon_1 \alpha^2$ \\
  \multicolumn{2}{c}{} & $($ & $-\frac{1}{\epsilon_1}\hphantom{{}-\epsilon_1 \alpha}$ & $-\frac{\alpha}{\epsilon_1}\hphantom{{}-\epsilon_1 \alpha^2}$ \\ \midrule
  $\frac{\alpha}{\epsilon_1}+\frac{\epsilon_1\alpha^3}{2}$ & $1$ & $($ & $-\frac{\alpha}{\epsilon_1}-\epsilon_1 \alpha^2$ & $-\frac{\epsilon_1\alpha^3}{2}$ \\
  \multicolumn{2}{c}{} & $($ & $-\frac{\alpha}{\epsilon_1}-\epsilon_1 \alpha^2$ & $-\frac{\epsilon_1\alpha^3}{2}$ \\
  \multicolumn{2}{c}{} & $($ & $-\frac{\alpha}{\epsilon_1}\hphantom{{}-\epsilon_1 \alpha^2}$ & $\frac{\epsilon_1\alpha^3}{2}$ \\ \midrule
  \multicolumn{5}{c}{Return $(\epsilon_1' = \frac{\epsilon_1\alpha^3}{2})$.} \\
  \bottomrule
\end{tabular}
\end{table}

The training data that executes this plan is given in \hyperref[tab:rinputF]{Table~\ref{tab:rinputF}}. Note that although the returned $\epsilon_1'$ is not a power of $\alpha$, we can use two additional coordinates and the following sequence of API calls to provide such a guarantee:
\begin{itemize}
  \item \rinputF{i_1}{\epsilon_1}, which returns $(\epsilon_1')$
  \item \rcopy{i_1}{i_2}{i_3}{\epsilon_1'}, which returns $(\epsilon_2, \epsilon_3)$
  \item \rreset{i_3}{\epsilon_3}
  \item \rcopy{i_2}{i_1}{i_3}{\epsilon_2}, which returns $(\epsilon_1'', \epsilon_3')$
  \item \rreset{i_3}{\epsilon_3'}
\end{itemize}
Of course, we need to remember to decrease the various $\epsilon$ parameters while other operations are running, to account for weight decay.

\subsection{Implementation of \rsetiftrue{i_1}{i_2}{\epsilon_1}}

At a high level, we mimic the implementation of \rinputF{i_1}{\epsilon_1}, but piggyback on a threshold check to read the first weight.

\begin{table}
\caption{Training data for \rsetiftrue{i_1}{i_2}{\epsilon_1}.}
\label{tab:rsetiftrue}
\centering
\begin{tabular}{@{}cccr@{}r@{, }r@{$) \to ($}r@{, }r@{$)$}@{}}
  \toprule
  $x_{i_1}$ & $x_{i_2}$ & $y$ & \multicolumn{5}{c}{Effect on $(w_{i_1}, w_{i_2})$} \\ \midrule
  $\left(-\frac{1}{\epsilon_1}-\epsilon_1 \alpha\right)$ & $1$ & $1$ & $($ & $-\epsilon_1$ & $0$ & $-\epsilon_1\alpha$ & $0$ \\
  \multicolumn{3}{c}{} & $($ & $\epsilon_1$ & $0$ & $-\frac{1}{\epsilon_1}\hphantom{{}-\epsilon_1\alpha}$ & $1$ \\ \midrule
  $-\frac{\alpha}{\epsilon_1}$ & $0$ & $1$ & $($ & $-\epsilon_1\alpha$ & $0$ & $-\frac{\alpha}{\epsilon_1}-\epsilon_1\alpha^2$ & $0$ \\
  \multicolumn{3}{c}{} & $($ & $-\frac{1}{\epsilon_1}\hphantom{{}+\epsilon_1\alpha}$ & $1$ & $-\frac{\alpha}{\epsilon_1}\hphantom{{}-\epsilon_1\alpha^2}$ & $\alpha$ \\ \midrule
  $\frac{\alpha^2}{\epsilon_1}+\frac{\epsilon_1\alpha^3}{2}$ & $0$ & $1$ & $($ & $-\frac{\alpha}{\epsilon_1}-\epsilon_1\alpha^2$ & $0$  & $-\frac{\epsilon_1\alpha^3}{2}$ & $0$ \\
  \multicolumn{3}{c}{} & $($ & $-\frac{\alpha}{\epsilon_1}\hphantom{{}-\epsilon_1\alpha^2}$ & $\alpha$ & $\frac{\epsilon_1\alpha^3}{2}$ & $\alpha^2$ \\ \midrule
  \multicolumn{8}{c}{Return $(\epsilon_1' = \frac{\epsilon_1 \alpha^3}{2}, \epsilon_2 = \alpha^2)$.} \\
  \bottomrule
\end{tabular}
\end{table}

The training data that executes this plan is given in \hyperref[tab:rsetiftrue]{Table~\ref{tab:rsetiftrue}}. Again, the returned $\epsilon_1'$ is not a power of $\alpha$, but  we can correct this with two additional coordinates and copying around values, as before.

\section{Proof Extensions for Additional Models}
\label{apx:additional}

In this appendix, we show how to extend our proofs to work for two additional, more complex models. In the first (easier) model, we consider a network with a single dense layer followed by a ReLU activation (dense-ReLU); the output of this network is compared against the training output using squared loss. In the second (harder) model, we consider a network with a dense layer followed by a ReLU activation followed by another dense layer (dense-ReLU-dense); the output of this network is also evaluated against the training output using squared loss.

\subsection{Dense-ReLU under Squared Loss}

Written in terms of the training example and weights, our network has the following loss function (note that we only have a single hidden node).
\[
\ell_{DR}(\bfw^t,(\bfx^t,y^t))= (y^t-\sigma(\bfw^t\cdot \bfx^t))^2
\]
where $\sigma(\cdot)$ is the coordinate-wise ReLU activation. At a fixed iteration, on a given example, the partial derivative\footnote{Notice that the derivative of $\sigma(0)$ is undefined, so our gadgets never result in a zero input to the ReLU activation unit.} with respect to the one weight $w_i$ at that step is:
\begin{align*}
  \frac{\partial \ell_{DR}(\bfw, \bfx, y)}{\partial w_i}
  &=
  \begin{cases}
    2(\bfw \cdot \bfx-y)x_i & \text{if } \bfw \cdot \bfx > 0 \\
    0 & \text{if } \bfw \cdot \bfx <0
  \end{cases}
\end{align*}

\begin{theorem}
  There is a reduction which, given a circuit $\C$ and a target binary string $s^*$, produces a set of training examples for OGD (where the updates are based on the $\ell_{DR}$ loss function) such that repeated application of $\C$ to the all-false string eventually produces the string $s^*$ if and only if OGD beginning with the all-zeroes weight vector and repeatedly fed this set of training examples (in the same order)  eventually produces a weight vector $\bfw^t$ with positive first coordinate.
\end{theorem}

\newcommand{\specq}{\Join}

The proof is the same as that of \hyperref[th:main]{Theorem~\ref{th:main}}, except we use the modified API found in \hyperref[tab:public-api-sq1]{Table~\ref{tab:public-api-sq1}}. As a consequence of using this modified API, we keep an additional special coordinate, $\specq$, denoting the fourth coordinate whose weight is $+1$ in between calls to our API. When we invoke \texttt{destructive\_nand} or \texttt{set\_false\_if\_unset}, we pass the fourth or second argument, respectively, to be $\specq$.

\begin{table}
\caption{Modified API for Dense-ReLU under Squared Loss.}
\label{tab:public-api-sq1}
\centering
\begin{tabular}{@{}lll@{}}
  \toprule
  \textbf{Function} & \textbf{Precondition(s)} & \textbf{Description} \\ \midrule
  % reset(i_1)
  \reset{i_1} &
  $i_1 \in \memory$ &
  $w_{i_1} \leftarrow 0$ \\
(for implementation, see  \hyperref[tab:reset1]{Table~\ref{tab:reset1}})&
  $w_{i_1} \in \binary$ &
  \\ \midrule
  \fnot{i_1} &
  $i_1 \in \memory$ &
  If $w_{i_1} == -1$, $w_{i_1} \leftarrow +1$ \\
 (for implementation, see  \hyperref[tab:not1]{Table~\ref{tab:not1}}) &
  $w_{i_1} \in \binary$ &
  If $w_{i_1} == +1$, $w_{i_1} \leftarrow -1$
  \\ \midrule
  % copy(i_1, i_2)
  \cpy{i_1}{i_2} &
  $i_1, i_2 \in \memory$ &
  $w_{i_2} \leftarrow w_{i_1}$ \\
  (for implementation, see  \hyperref[tab:copy1]{Table~\ref{tab:copy1}})&
  $w_{i_1} \in \binary$ &
  \\
  &
  $w_{i_2} = 0$ &
  \\ \midrule
  % destructive_nand(i_1, i_2, i_3)
  {\texttt{destructive\_nand}}$(i_1,i_2,i_3,i_4)$ &
  $i_1, i_2, i_3,i_4 \in \memory$ &
  $w_{i_3} \leftarrow \text{NAND}(w_{i_1}, w_{i_2})$ \\
  (for implementation, see  \hyperref[tab:dnand1]{Table~\ref{tab:dnand1}})&
  $w_{i_1} \in \binary$ &
  $w_{i_1} \leftarrow 0$ \\
  &
  $w_{i_2} \in \binary$ &
  $w_{i_2} \leftarrow 0$ \\
  &
  $w_{i_3} = 0$ & $w_{i_4}\leftarrow +1$ \\
  &
  $w_{i_4} = +1$
  \\ \midrule
  % initialize_false(i_1)
  {\texttt{set\_false\_if\_unset}}$(i_1,i_2)$ &
  $i_1, i_2 \in \memory$ &
  If $w_{i_1} == 0$, $w_{i_1} \leftarrow -1$ \\
 (for implementation, see  \hyperref[tab:set1]{Table~\ref{tab:set1}}) &
  $w_{i_1} \in \ubinary$ & $w_{i_2} \leftarrow +1$
  \\ 
  &
  $w_{i_2} = +1$ &
  \\ \midrule
  % set_if_true(i_1, i_2)
  \setiftrue{i_1}{i_2} &
  $i_1, i_2 \in \memory$ &
  If $w_{i_1} > 0$, $w_{i_2} \leftarrow +1$ \\
 (for implementation, see  \hyperref[tab:trset1]{Table~\ref{tab:trset1}}) &
  $w_{i_1} \in \binary$ &
  If $w_{i_1} < 0$, $w_{i_2}$ remains at $0$ \\
  &
  $w_{i_2} = 0$ &
  (including in intermediate steps) \\
  \bottomrule
\end{tabular}
\end{table}

\begin{table}
\caption{Training data for {\texttt{reset}$(i_1)$} for Dense-ReLU under Squared Loss.}
\label{tab:reset1}
\centering
\begin{tabular}{@{}ccr@{}r@{$) \to ($}r@{$)$}@{}}
  \toprule
  $x_{i_1}$ & $y$ & \multicolumn{3}{c}{Effect on $(w_{i_1})$} \\ \midrule
  $1$ & $0$ & $($ & $-1$ & $-1$ \\
  \multicolumn{2}{c}{} & $($ & $1$ & $-1$ \\ \midrule
  $-1$ & $\tfrac12$ & $($ & $-1$ & $0$ \\
  \multicolumn{2}{c}{} & $($ & $-1$ & $0$ \\ \midrule
    \bottomrule
\end{tabular}
\end{table}

\begin{table}
\caption{Training data for \fnot{i_1} for Dense-ReLU under Squared Loss.}
\label{tab:not1}
\centering
\begin{tabular}{@{}ccr@{}r@{$) \to ($}r@{$)$}@{}}
  \toprule
  $x_{i_1}$ & $y$ & \multicolumn{3}{c}{Effect on $(w_{i_1})$} \\ \midrule
  $1$ & $-2$ & $($ & $-1$ & $-1$ \\
  \multicolumn{2}{c}{} & $($ & $1$ & $-5$ \\ \midrule
  $-\tfrac{1}{2}$ & $-\tfrac32$ & $($ & $-1$ & $1$ \\
  \multicolumn{2}{c}{} & $($ & $-5$ & $-1$ \\ \midrule
    \bottomrule
\end{tabular}
\end{table}

\begin{table}
\caption{Training data for {\texttt{copy}}($i_1,i_2$) for Dense-ReLU under Squared Loss.}%\rsetiftrue{i_1}{i_2}{\epsilon_1}.}
\label{tab:copy1}
\centering
\begin{tabular}{@{}cccr@{}r@{, }r@{$) \to ($}r@{, }r@{$)$}@{}}
  \toprule
  $x_{i_1}$ & $x_{i_2}$ & $y$ & \multicolumn{5}{c}{Effect on $(w_{i_1}, w_{i_2})$} \\ \midrule
  $1$ & $-1$ & $\tfrac78$ & $($ & $-1$ & $0$ & $-1$ & $0$ \\
  \multicolumn{3}{c}{} & $($ & $1$ & $0$ & $\tfrac34$ & $\tfrac14$ \\ \midrule
  $-1$ & $1$ & $\tfrac78$ & $($ & $-1$ & $0$ & $-\tfrac34$ & $-\tfrac14$ \\
  \multicolumn{3}{c}{} & $($ & $\tfrac34$ & $\tfrac14$ & $\tfrac34$ & $\tfrac14$ \\ \midrule
  $-1$ & $0$ & $\tfrac78$ & $($ & $-\tfrac34$ & $-\tfrac14$  & $-1$ & $-\tfrac14$ \\
  \multicolumn{3}{c}{} & $($ & $\tfrac34$ & $\tfrac14$ & $\tfrac34$ & $\tfrac14$ \\ \midrule
  $1$ & $0$ & $\tfrac78$ & $($ & $-1$ & $-\tfrac14$  &  $-1$ & $-\tfrac14$ \\
  \multicolumn{3}{c}{} & $($ & $\tfrac34$ & $\tfrac14$ & $1$ & $\tfrac14$ \\ \midrule
  $0$ & $-1$ & $\tfrac58$ & $($ & $-1$ & $-\tfrac14$  & $-1$ & $-1$ \\
  \multicolumn{3}{c}{} & $($ & $1$ & $\tfrac14$ & $1$ & $\tfrac14$ \\ \midrule
  $0$ & $1$ & $\tfrac58$ & $($ & $-1$ & $-1$  & $-1$ & $-1$ \\
  \multicolumn{3}{c}{} & $($ & $1$ & $\tfrac14$ & $1$ & $1$ \\ \midrule
 % \multicolumn{8}{c}{Return $(\epsilon_1' = \frac{\epsilon_1 \alpha^3}{2}, \epsilon_2 = \alpha^2)$.} \\
  \bottomrule
\end{tabular}
\end{table}

\begin{table}
\small
\caption{Training data for \texttt{destructive\_nand}($i_1,i_2,i_3,i_4$) for Dense-ReLU under Squared Loss.}%\rdnand{i_1}{i_2}{i_3}{\epsilon_1}{\epsilon_2}.}
\label{tab:dnand1}
\centering
%\begin{tabular}{@{}cccccr@{}r@{, }r@{, }r@{$) \to ($}r@{, }r@{, }r@{$)$}@{}}
\begin{tabular}{@{}cccccr@{}r@{, }r@{, }r@{, }r@{$) \to ($}r@{, }r@{, }r@{, }r@{$)$}@{}}
  \toprule
  $x_{i_1}$ & $x_{i_2}$ & $x_{i_3}$ & $x_{i_4}$& $y$ & \multicolumn{9}{c}{Effect on $(w_{i_1}, w_{i_2}, w_{i_3}, w_{i_4})$} \\ \midrule
  $-1$ & $0$ & $0$ & $0$ & $\tfrac32$ & $($ & $-1$ & $-1$ & $0$& $1$  & $-2$ & $-1$ & $0$& $1$  \\
  \multicolumn{5}{c}{} & $($ & $-1$ & $1$& $0$ & $1$  & $-2$ & $1$ & $0$ & $1$ \\
  \multicolumn{5}{c}{} & $($ & $1$& $-1$ & $0$  & $1$ & $1$ & $-1$ & $0$ & $1$ \\
  \multicolumn{5}{c}{} & $($ & $1$ & $1$ & $0$ & $1$ & $1$ & $1$ & $0$ & $1$ \\ \midrule
  $0$ & $-1$& $0$ & $0$ & $\tfrac32$ & $($ & $-2$ & $-1$ & $0$ & $1$ & $-2$& $-2$  & $0$ & $1$ \\
  \multicolumn{5}{c}{} & $($ & $-2$ & $1$ & $0$ & $1$ & $-2$ & $1$ & $0$ & $1$ \\
  \multicolumn{5}{c}{} & $($ & $1$ & $-1$ & $0$ & $1$ & $1$ & $-2$ & $0$ & $1$ \\
  \multicolumn{5}{c}{} & $($ & $1$& $1$& $0$  & $1$ & $1$ & $1$ & $0$ & $1$ \\ \midrule
  $1$ & $1$& $1$ & $0$ & $\tfrac12$ & $($ & $-2$ & $-2$ & $0$ & $1$ & $-2$& $-2$  & $0$ & $1$ \\
  \multicolumn{5}{c}{} & $($ & $-2$ & $1$ & $0$ & $1$ &$-2$ & $1$ & $0$ & $1$ \\
  \multicolumn{5}{c}{} & $($ & $1$ & $-2$ & $0$ & $1$ &$1$ & $-2$ & $0$ & $1$ \\
  \multicolumn{5}{c}{} & $($ & $1$& $1$& $0$  & $1$ & $-2$ & $-2$ & $-3$ & $1$ \\ \midrule
 $-1$ & $0$& $0$ & $0$ & $\tfrac12$ & $($ & $-2$ & $-2$ & $0$ & $1$ & $1$& $-2$  & $0$ & $1$ \\
  \multicolumn{5}{c}{} & $($ & $-2$ & $1$ & $0$ & $1$ & $1$ & $1$ & $0$ & $1$ \\
  \multicolumn{5}{c}{} & $($ &  $1$ & $-2$ & $0$ & $1$ & $1$ & $-2$ & $0$ & $1$ \\
  \multicolumn{5}{c}{} & $($ & $-2$& $-2$& $-3$  & $1$ & $1$ & $-2$ & $-3$ & $1$ \\ \midrule
$1$ & $0$& $0$ & $0$ & $\tfrac12$ & $($ & $1$ & $-2$ & $0$ & $1$ & $0$& $-2$  & $0$ & $1$ \\
  \multicolumn{5}{c}{} & $($ & $1$ & $1$ & $0$ & $1$ & $0$ & $1$ & $0$ & $1$ \\
  \multicolumn{5}{c}{} & $($ & $1$ & $-2$ & $0$ & $1$ & $0$ & $-2$ & $0$ & $1$ \\
  \multicolumn{5}{c}{} & $($ & $1$& $-2$& $-3$  & $1$ & $0$ & $-2$ & $-3$ & $1$ \\ \midrule
$0$ & $-1$& $0$ & $0$ & $\tfrac12$ & $($ & $0$ & $-2$ & $0$ & $1$ & $0$& $1$  & $0$ & $1$ \\
  \multicolumn{5}{c}{} & $($ & $0$ & $1$ & $0$ & $1$ &$0$& $1$  & $0$ & $1$ \\
  \multicolumn{5}{c}{} & $($ & $0$ & $-2$ & $0$ & $1$ & $0$& $1$  & $0$ & $1$ \\
  \multicolumn{5}{c}{} & $($ & $0$& $-2$& $-3$  & $1$ & $0$& $1$  & $-3$ & $1$ \\ \midrule
$0$ & $1$& $0$ & $0$ & $\tfrac12$ & $($ & $0$ & $1$ & $0$ & $1$ & $0$& $0$  & $0$ & $1$ \\
  \multicolumn{5}{c}{} & $($ & $0$ & $1$ & $0$ & $1$ & $0$& $0$  & $0$ & $1$ \\
  \multicolumn{5}{c}{} & $($ & $0$ & $1$ & $0$ & $1$ & $0$& $0$  & $0$ & $1$ \\
  \multicolumn{5}{c}{} & $($ & $0$ & $1$ & $-3$ & $1$ & $0$& $0$  & $-3$ & $1$ \\ \midrule
$0$ & $0$& $1$ & $1$ & $-\tfrac32$ & $($ & $0$ & $0$ & $0$ & $1$ & $0$& $0$  & $-5$ & $-4$ \\
  \multicolumn{5}{c}{} & $($ & $0$ & $0$ & $0$ & $1$ & $0$& $0$  & $-5$ & $-4$ \\
  \multicolumn{5}{c}{} & $($ & $0$ & $0$ & $0$ & $1$ & $0$& $0$  & $-5$ & $-4$ \\
  \multicolumn{5}{c}{} & $($ & $0$ & $0$ & $-3$ & $1$ & $0$& $0$  & $-3$ & $1$ \\ \midrule
$0$ & $0$& $-1$ & $0$ & $2$ & $($ & $0$ & $0$ & $-5$ & $-4$ & $0$& $0$  & $1$ & $-4$ \\
  \multicolumn{5}{c}{} & $($ & $0$ & $0$ & $-5$ & $-4$ &$0$& $0$  & $1$ & $-4$ \\
  \multicolumn{5}{c}{} & $($ & $0$ & $0$ & $-5$ & $-4$ & $0$& $0$  & $1$ & $-4$ \\
  \multicolumn{5}{c}{} & $($ &$0$ & $0$ & $-3$ & $1$ & $0$& $0$  & $-1$ & $1$\\ \midrule
$0$ & $0$& $0$ & $-1$ & $\tfrac32$ & $($ & $0$ & $0$ & $1$ & $-4$ & $0$& $0$  & $1$ & $1$ \\
  \multicolumn{5}{c}{} & $($ & $0$ & $0$ & $1$ & $-4$ & $0$& $0$  & $1$ & $1$ \\
  \multicolumn{5}{c}{} & $($ & $0$ & $0$ & $1$ & $-4$ & $0$& $0$  & $1$ & $1$ \\
  \multicolumn{5}{c}{} & $($ & $0$ & $0$ & $-1$ & $1$ & $0$& $0$  & $-1$ & $1$ \\ \midrule 
   \bottomrule
\end{tabular}
\end{table}

\begin{table}
\caption{Training data for {\texttt{set\_false\_if\_unset}}($i_1,i_2$) for Dense-ReLU under Squared Loss.}%\rsetiftrue{i_1}{i_2}{\epsilon_1}.}
\label{tab:set1}
\centering
\begin{tabular}{@{}cccr@{}r@{, }r@{$) \to ($}r@{, }r@{$)$}@{}}
  \toprule
  $x_{i_1}$ & $x_{i_2}$ & $y$ & \multicolumn{5}{c}{Effect on $(w_{i_1}, w_{i_2})$} \\ \midrule
  $1$ & $\tfrac12$ & $0$ & $($ & $-1$ & $1$ & $-1$ & $1$ \\
  \multicolumn{3}{c}{} & $($ & $0$ & $1$ & $-1$ & $\tfrac12$ \\
  \multicolumn{3}{c}{} & $($ & $1$ & $1$ & $-2$ & $-\tfrac12$ \\ \midrule
  $0$ & $1$ & $0$ & $($ & $-1$ & $1$ & $-1$ & $-1$ \\
  \multicolumn{3}{c}{} & $($ & $-1$ & $\tfrac12$ & $-1$ & $-\tfrac12$ \\
  \multicolumn{3}{c}{} & $($ & $-2$ & $-\tfrac12$ & $-2$ & $-\tfrac12$ \\ \midrule
  $0$ & $-2$ & $\tfrac32$ & $($ & $-1$ & $-1$ & $-1$ & $1$ \\
  \multicolumn{3}{c}{} & $($ & $-1$ & $-\tfrac12$ & $-1$ & $-\tfrac52$ \\
  \multicolumn{3}{c}{} & $($ & $-2$ & $-\tfrac12$ & $-2$ & $-\tfrac52$ \\ \midrule
  $0$ & $-1$ & $\tfrac34$ & $($ & $-1$ & $1$ & $-1$ & $1$ \\
  \multicolumn{3}{c}{} & $($ & $-1$ & $-\tfrac52$ & $-1$ & $1$ \\
  \multicolumn{3}{c}{} & $($ & $-2$ & $-\tfrac52$ & $-2$ & $1$ \\ \midrule
  $-1$ & $0$ & $\tfrac34$ & $($ & $-1$ & $1$ & $-\tfrac12$ & $1$ \\
  \multicolumn{3}{c}{} & $($ & $-1$ & $1$ & $-\tfrac12$ & $1$ \\
  \multicolumn{3}{c}{} & $($ & $-2$ & $1$ & $\tfrac12$ & $1$ \\ \midrule
  $-1$ & $0$ & $\tfrac34$ & $($ & $-\tfrac12$ & $1$ & $-1$ & $1$ \\
  \multicolumn{3}{c}{} & $($ & $-\tfrac12$ & $1$ & $-1$ & $1$ \\
  \multicolumn{3}{c}{} & $($ & $\tfrac12$ & $1$ & $\tfrac12$ & $1$ \\ \midrule
  $1$ & $0$ & $\tfrac34$ & $($ & $-1$ & $1$  & $-1$ & $1$ \\
  \multicolumn{3}{c}{} & $($ & $-1$ & $1$ & $-1$ & $1$ \\
  \multicolumn{3}{c}{} & $($ & $\tfrac12$ & $1$ & $1$ & $1$ \\
  \bottomrule
\end{tabular}
\end{table}

\begin{table}
\caption{Training data for {\texttt{copy\_if\_true}}($i_1,i_2$) for Dense-ReLU under Squared Loss.}%\rsetiftrue{i_1}{i_2}{\epsilon_1}.}
\label{tab:trset1}
\centering
\begin{tabular}{@{}cccr@{}r@{, }r@{$) \to ($}r@{, }r@{$)$}@{}}
  \toprule
  $x_{i_1}$ & $x_{i_2}$ & $y$ & \multicolumn{5}{c}{Effect on $(w_{i_1}, w_{i_2})$} \\ \midrule
  $1$ & $-2$ & $\tfrac34$ & $($ & $-1$ & $0$ & $-1$ & $0$ \\
  \multicolumn{3}{c}{} & $($ & $1$ & $0$ & $\tfrac12$ & $1$ \\ \midrule
  $1$ & $0$ & $\tfrac34$ & $($ & $-1$ & $0$ & $-1$ & $0$ \\
  \multicolumn{3}{c}{} & $($ & $\tfrac12$ & $1$ & $1$ & $1$ \\
  \bottomrule
\end{tabular}
\end{table}

%\begin{tabular}{@{}ccr@{}r@{$) \to ($}r@{$)$}@{}}
%\begin{tabular}{@{}cccr@{}r@{, }r@{$) \to ($}r@{, }r@{$)$}@{}}
%\begin{tabular}{@{}ccccr@{}r@{, }r@{, }r@{$) \to ($}r@{, }r@{, }r@{$)$}@{}}
%\begin{tabular}{@{}cccccr@{}r@{, }r@{, }r@{, }r@{$) \to ($}r@{, }r@{, }r@{, }r@{$)$}@{}}
%\begin{tabular}{@{}ccccccr@{}r@{, }r@{, }r@{, }r@{, }r@{$) \to ($}r@{, }r@{, }r@{, }r@{, }r@{$)$}@{}}

\subsection{Dense-ReLU-Dense under Squared Loss}
Having an additional layer gives us the following loss function.
\[
\ell_{DRD}((\bfw^t,v^t),(\bfx^t,y^t))= (y^t-v^t\sigma(\bfw^t\cdot \bfx^t))^2
\]
where, as before, $\sigma(\cdot)$ denotes a ReLU activation function. At a fixed iteration, on a given example, the partial derivative w. r. t. the weight  $(\bfw, v)$ at that step is:
\begin{align*}
  \frac{\partial \ell_{DRD}(\bfw,v, \bfx, y)}{\partial w_i}
  &=
  \begin{cases}
   2(v\bfw \cdot \bfx-y)x_iv & \text{if } \bfw \cdot \bfx > 0 \\
    0 & \text{if } \bfw \cdot \bfx <0
  \end{cases}
 \end{align*}
 \begin{align*}
  \frac{\partial \ell_{DRD}(\bfw,v, \bfx, y)}{\partial v}
  &=
  \begin{cases}
    2(v\bfw \cdot \bfx-y)\bfw \cdot \bfx & \text{if } \bfw \cdot \bfx > 0 \\
    0 & \text{if } \bfw \cdot \bfx <0
  \end{cases}
\end{align*}

\begin{theorem}
  There is a reduction which, given a circuit $\C$ and a target binary string $s^*$, produces a set of training examples for OGD (where the updates are based on the $\ell_{DRD}$ loss function) such that repeated application of $\C$ to the all-false string eventually produces the string $s^*$ if and only if OGD beginning with the all-zeroes weight vector and repeatedly fed this set of training examples (in the same order)  eventually produces a weight vector $\bfw^t$ with positive first coordinate.
\end{theorem}

Again, the proof is the same as that of \hyperref[th:main]{Theorem~\ref{th:main}}, except we use the modified API found in \hyperref[tab:public-api-sq2]{Table~\ref{tab:public-api-sq2}}. Just as in the previous model, we need to keep an additional special coordinate, $\specq$, denoting the fourth coordinate whose weight is $+1$ in between calls to our API. Whenever we invoke any method of our API, we pass it $\specq$ as its final argument. The other big difference for this case is we have an additional (scalar) weight variable $v$ representing the sole weight in the second layer of our network. Before and after any method of our API, we require $v$ to be one and ensure that $v$ is one again. Modulo this requirement, the idea behind all of our gadgets is essentially the same as the previous section; at a high level we simply insert additional training points to correct the special coordinate $\specq$ and the second-layer weight $v$ to one between every previous pair of training points.

\begin{table}[h]
\caption{Augmented API for Dense-ReLU-Dense under Squared Loss.}
\label{tab:public-api-sq2}
\centering
\begin{tabular}{@{}lll@{}}
  \toprule
  \textbf{Function} & \textbf{Precondition(s)} & \textbf{Description} \\ \midrule
  % reset(i_1)
  {\texttt{reset}}$(i_1,i_2)$ &
  $i_1,i_2 \in \memory$ &
  $w_{i_1} \leftarrow 0$ \\
  (for implementation, see  \hyperref[tab:reset2]{Table~\ref{tab:reset2}})&
  $w_{i_1} \in \binary$ & $w_{i_2}\leftarrow +1$\\
  &
  $w_{i_2}=+1,v=+1$ & $v\leftarrow +1$
  \\ \midrule
  {\texttt{not}}$(i_1,i_2)$ &
  $i_1,i_2 \in \memory$ &
  If $w_{i_1} == -1$, $w_{i_1} \leftarrow +1$\\
  (for implementation, see  \hyperref[tab:not2]{Table~\ref{tab:not2}})&
  $w_{i_1} \in \binary$ & If $w_{i_1} == +1, w_{i_1} \leftarrow -1$\\
  &
  $w_{i_2}=+1,v=+1$ & $w_{i_2}\leftarrow +1,v\leftarrow +1$
  \\ \midrule
  % copy(i_1, i_2)
  {\texttt{copy}}$(i_1,i_2,i_3)$ &
  $i_1,i_2,i_3 \in \memory$ &
  $w_{i_2} \leftarrow w_{i_1}$\\
 (for implementation, see  \hyperref[tab:copy2]{Table~\ref{tab:copy2}}) &
  $w_{i_1} \in \binary$ & $w_{i_1}$ remains unchanged\\
  &
  $w_{i_2}=0,w_{i_3}=+1,v=+1$ & $w_{i_3}\leftarrow +1,v\leftarrow +1$
  \\ \midrule
  % destructive_nand(i_1, i_2, i_3)
  {\texttt{destructive\_nand}}$(i_1,i_2,i_3,i_4)$ &
  $i_1, i_2, i_3,i_4 \in \memory$ &
  $w_{i_3} \leftarrow \text{NAND}(w_{i_1}, w_{i_2})$ \\
  (for implementation, see  \hyperref[tab:dnand2]{Table~\ref{tab:dnand2}})&
  $w_{i_1} \in \binary$ &
  $w_{i_1} \leftarrow 0$ \\
  &
  $w_{i_2} \in \binary$ &
  $w_{i_2} \leftarrow 0$ \\
  &
  $w_{i_3} = 0$ & $w_{i_4}\leftarrow +1$ \\
  &
  $w_{i_4} = +1, v= +1$&$v\leftarrow +1$
  \\ \midrule
  % initialize_false(i_1)
  {\texttt{set\_false\_if\_unset}}$(i_1,i_2)$ &
  $i_1, i_2 \in \memory$ &
  If $w_{i_1} == 0$, $w_{i_1} \leftarrow -1$ \\
  (for implementation, see  \hyperref[tab:setF2]{Table~\ref{tab:setF2}})&
  $w_{i_1} \in \ubinary$ & $w_{i_2} \leftarrow +1$
  \\ 
  &
  $w_{i_2} = +1, v = +1$ & $v \leftarrow +1$
  \\ \midrule
  % set_if_true(i_1, i_2)
  {\texttt{copy\_if\_true}}$(i_1,i_2,i_3)$ &
  $i_1, i_2,i_3 \in \memory$ &
  If $w_{i_1} > 0$, $w_{i_2} \leftarrow +1$ \\
  (for implementation, see  \hyperref[tab:setT2]{Table~\ref{tab:setT2}})&
  $w_{i_1} \in \binary$ &
  If $w_{i_1} < 0$, $w_{i_2}$ remains at $0$ \\
  &
  $w_{i_2} = 0$ &(including in intermediate steps)  \\
  & $w_{i_3} = +1,v = +1$ & $w_{i_3} \leftarrow +1,v \leftarrow +1$\\
  \bottomrule
\end{tabular}
\end{table}

\begin{table}
\caption{Training data for {\texttt{reset}}($i_1,i_2$) for Dense-ReLU-Dense under Squared Loss.}%\rsetiftrue{i_1}{i_2}{\epsilon_1}.}
\label{tab:reset2}
\centering
\begin{tabular}{@{}cccr@{}r@{, }r@{, }r@{$) \to ($}r@{, }r@{, }r@{$)$}@{}}
  \toprule
  $x_{i_1}$ & $x_{i_2}$ & $y$ & \multicolumn{7}{c}{Effect on $(w_{i_1}, w_{i_2}, v)$} \\ \midrule
  $1$ & $0$ & $\tfrac34$ & $($ & $-1$ & $1$ & $1$& $-1$& $1$ & $1$ \\
  \multicolumn{3}{c}{} & $($ & $1$& $1$& $1$ & $\tfrac{1}{2}$ & $1$ & $\tfrac{1}{2}$ \\ \midrule
  $0$ & $1$ & $1$ & $($ & $-1$ & $1$ & $1$& $-1$& $1$ & $1$ \\
  \multicolumn{3}{c}{} & $($ & $\tfrac{1}{2}$& $1$& $\tfrac{1}{2}$ & $\tfrac{1}{2}$ & $\tfrac{3}{2}$ & $\tfrac{3}{2}$ \\ \midrule
  $0$ & $1$ & $\tfrac{17}{4}$ & $($ & $-1$ & $1$ & $1$& $-1$& $\tfrac{15}{2}$ & $\tfrac{15}{2}$ \\
  \multicolumn{3}{c}{} & $($ & $\tfrac{1}{2}$& $\tfrac{3}{2}$& $\tfrac{3}{2}$ & $\tfrac{1}{2}$ & $\tfrac{15}{2}$ & $\tfrac{15}{2}$ \\ \midrule
  $0$ & $\tfrac{2}{15}$ & $\tfrac{17}{4}$ & $($ & $-1$ & $\tfrac{15}{2}$ & $\tfrac{15}{2}$& $-1$& $1$ & $1$ \\
  \multicolumn{3}{c}{} & $($ & $\tfrac{1}{2}$& $\tfrac{15}{2}$& $\tfrac{15}{2}$ & $\tfrac{1}{2}$ & $1$ & $1$ \\ \midrule
  $1$ & $0$ & $-\tfrac{1}{4}$ & $($ & $-1$ & $1$ & $1$& $-1$& $1$ & $1$ \\
  \multicolumn{3}{c}{} & $($ & $\tfrac{1}{2}$& $1$& $1$ & $-1$ & $1$ & $\tfrac14$ \\ \midrule
  $0$ & $1$ & $\tfrac{3}{4}$ & $($ & $-1$ & $1$ & $1$& $-1$& $\tfrac{1}{2}$ & $\tfrac{1}{2}$ \\
  \multicolumn{3}{c}{} & $($ & $-1$& $1$& $\tfrac{1}{4}$ & $-1$ & $\tfrac{5}{4}$ & $\tfrac{5}{4}$ \\ \midrule
  $0$ & $1$ & $\tfrac{31}{16}$ & $($ & $-1$ & $\tfrac{1}{2}$ & $\tfrac{1}{2}$& $-1$& $\tfrac{35}{16}$ & $\tfrac{35}{16}$ \\
  \multicolumn{3}{c}{} & $($ & $-1$& $\tfrac{5}{4}$& $\tfrac{5}{4}$ & $-1$ & $\tfrac{35}{16}$ & $\tfrac{35}{16}$ \\ \midrule
  $0$ & $\tfrac{16}{35}$ & $\tfrac{51}{32}$ & $($ & $-1$ & $\tfrac{35}{16}$ & $\tfrac{35}{16}$& $-1$& $1$ & $1$ \\
  \multicolumn{3}{c}{} & $($ & $-1$& $\tfrac{35}{16}$& $\tfrac{35}{16}$ & $-1$ & $1$ & $1$ \\ \midrule
  $-1$ & $0$ & $\tfrac{3}{4}$ & $($ & $-1$ & $1$ & $1$& $-\tfrac{1}{2}$& $1$ & $\tfrac{1}{2}$ \\
  \multicolumn{3}{c}{} & $($ & $-1$& $1$& $1$ & $-\tfrac{1}{2}$& $1$ & $\tfrac{1}{2}$ \\ \midrule
  $0$ & $1$ & $1$ & $($ & $-\tfrac{1}{2}$ & $1$ & $\tfrac{1}{2}$& $-\tfrac{1}{2}$& $\tfrac{3}{2}$ & $\tfrac{3}{2}$ \\
  \multicolumn{3}{c}{} & $($ & $-\tfrac{1}{2}$& $1$& $\tfrac{1}{2}$ & $-\tfrac{1}{2}$ & $\tfrac{3}{2}$ & $\tfrac{3}{2}$ \\ \midrule
  $0$ & $\tfrac23$ & $\tfrac{5}{4}$ & $($ & $-\tfrac{1}{2}$ & $\tfrac{3}{2}$ & $\tfrac{3}{2}$& $-\tfrac{1}{2}$& $1$ & $1$ \\
  \multicolumn{3}{c}{} & $($ & $-\tfrac{1}{2}$& $\tfrac{3}{2}$& $\tfrac{3}{2}$ & $-\tfrac{1}{2}$ & $1$ & $1$ \\ \midrule
  $-1$ & $0$ & $\tfrac{1}{4}$ & $($ & $-\tfrac{1}{2}$ & $1$ & $1$& $0$& $1$ & $\tfrac{3}{4}$ \\
  \multicolumn{3}{c}{} & $($ & $-\tfrac{1}{2}$& $1$& $1$ & $0$ & $1$ & $\tfrac{3}{4}$ \\ \midrule
  $0$ & $1$ & $\tfrac{5}{4}$ & $($ & $0$ & $1$ & $\tfrac{3}{4}$& $0$& $\tfrac{7}{4}$ & $\tfrac{7}{4}$ \\
  \multicolumn{3}{c}{} & $($ & $0$& $1$& $\tfrac{3}{4}$ & $0$ & $\tfrac{7}{4}$ & $\tfrac{7}{4}$ \\ \midrule
  $0$ & $\tfrac47$ & $\tfrac{11}{8}$ & $($ & $0$ & $\tfrac{7}{4}$ & $\tfrac{7}{4}$& $0$& $1$ & $1$ \\
  \multicolumn{3}{c}{} & $($ & $0$& $\tfrac{7}{4}$& $\tfrac{7}{4}$ & $0$ & $1$ & $1$ \\ \midrule
 % \multicolumn{8}{c}{Return $(\epsilon_1' = \frac{\epsilon_1 \alpha^3}{2}, \epsilon_2 = \alpha^2)$.} \\
  \bottomrule
\end{tabular}
\end{table}

\begin{table}
\caption{Training data for {\texttt{not}}($i_1,i_2$) for Dense-ReLU-Dense under Squared Loss.}%\rsetiftrue{i_1}{i_2}{\epsilon_1}.}
\label{tab:not2}
\centering
\begin{tabular}{@{}cccr@{}r@{, }r@{, }r@{$) \to ($}r@{, }r@{, }r@{$)$}@{}}
  \toprule
  $x_{i_1}$ & $x_{i_2}$ & $y$ & \multicolumn{7}{c}{Effect on $(w_{i_1}, w_{i_2}, v)$} \\ \midrule
  $1$ & $0$ & $-4$ & $($ & $-1$ & $1$ & $1$& $-1$ & $1$ & $1$ \\
  \multicolumn{3}{c}{} & $($ & $1$& $1$& $1$ & $-9$& $1$& $-9$ \\ \midrule
  $0$ & $1$ & $-\tfrac{17}{2}$ & $($ & $-1$ & $1$ & $1$& $-1$ & $-18$ & $-18$ \\
  \multicolumn{3}{c}{} & $($ & $-9$& $1$& $-9$ & $-9$& $-8$& $-8$ \\ \midrule
  $0$ & $-1$ & $-\tfrac{1063}{2}$ & $($ & $-1$ & $-18$ & $-18$& $-1$ & $-7488$ & $-7488$ \\
  \multicolumn{3}{c}{} & $($ & $-9$& $-8$& $-8$ & $-9$& $-7488$& $-7488$ \\ \midrule
  $0$ & $-\tfrac{1}{7488}$ & $-\tfrac{7487}{2}$ & $($ & $-1$ & $-7488$ & $-7488$& $-1$& $1$ & $1$ \\
  \multicolumn{3}{c}{} & $($ & $-9$& $-7488$& $-7488$ & $-9$& $1$& $1$ \\ \midrule
  $-\tfrac12$ & $0$ & $\tfrac32$ & $($ & $-1$ & $1$ & $1$& $-2$ & $1$ & $2$ \\
  \multicolumn{3}{c}{} & $($ & $-9$& $1$& $1$ & $-6$& $1$& $-26$ \\ \midrule
  $0$ & $1$ & $\tfrac52$ & $($ & $-2$ & $1$ & $2$& $-2$ & $3$ & $3$ \\
  \multicolumn{3}{c}{} & $($ & $-6$& $1$& $-26$ & $-6$& $-1481$& $31$ \\ \midrule
  $0$ & $-1$ & $\tfrac{91803}{2}$ & $($ & $-2$ & $3$ & $3$& $-2$ & $3$ & $3$ \\
  \multicolumn{3}{c}{} & $($ & $-6$& $-1481$& $31$ & $-6$& $-1450$& $-1450$ \\ \midrule
  $0$ & $-\tfrac{1}{1450}$ & $\tfrac{1447}{2}$ & $($ & $-2$ & $3$ & $3$& $-2$ & $3$ & $3$ \\
  \multicolumn{3}{c}{} & $($ & $-6$& $-1450$& $-1450$ & $-6$& $3$& $3$ \\ \midrule
  $0$ & $\tfrac13$ & $2$ & $($ & $-2$ & $3$ & $3$& $-2$ & $1$ & $1$ \\
  \multicolumn{3}{c}{} & $($ & $-6$& $3$& $3$ & $-6$& $1$& $1$ \\ \midrule
  $-\tfrac12$ & $0$ & $-2$ & $($ & $-2$ & $1$ & $1$& $1$ & $1$ & $-5$ \\
  \multicolumn{3}{c}{} & $($ & $-6$& $1$& $1$ & $-1$& $1$& $-29$ \\ \midrule
  $0$ & $1$ & $-\tfrac92$ & $($ & $1$ & $1$ & $-5$& $1$ & $-4$ & $-4$ \\
  \multicolumn{3}{c}{} & $($ & $-1$& $1$& $-29$ &  $-1$& $-1420$& $20$ \\ \midrule
  $0$ & $-1$ & $\tfrac{56799}{2}$ & $($ & $1$ & $-4$ & $-4$& $1$ & $227320$ & $227320$ \\
  \multicolumn{3}{c}{} & $($ & $-1$& $-1420$& $20$ & $-1$& $-1400$& $-1400$\\ \midrule
  $0$ & $\tfrac{1}{227320}$ & $112960$ & $($ & $1$ & $227320$ & $227320$& $1$ & $-1400$ & $-1400$ \\
  \multicolumn{3}{c}{} & $($ & $-1$& $-1400$& $-1400$ & $-1$& $-1400$& $-1400$ \\ \midrule
  $0$ & $-\tfrac{1}{1400}$ & $\tfrac{1399}{2}$ & $($ & $1$ & $-1400$ & $-1400$& $1$& $1$ & $1$ \\
  \multicolumn{3}{c}{} & $($ & $-1$& $-1400$& $-1400$ & $-1$ & $1$ & $1$ \\ \midrule
 % \multicolumn{8}{c}{Return $(\epsilon_1' = \frac{\epsilon_1 \alpha^3}{2}, \epsilon_2 = \alpha^2)$.} \\
  \bottomrule
\end{tabular}
\end{table}

\begin{table}
\small
\caption{Training data for \texttt{copy}($i_1,i_2,i_3$) for Dense-ReLU-Dense under Squared Loss (Part~1~of~2) $\left(\text{here } \rho=47\tfrac{1272583}{3125000}\right)$.}%\rdnand{i_1}{i_2}{i_3}{\epsilon_1}{\epsilon_2}.}
\label{tab:copy2}
\centering
%\begin{tabular}{@{}cccccr@{}r@{, }r@{, }r@{$) \to ($}r@{, }r@{, }r@{$)$}@{}}
\begin{tabular}{@{}ccccr@{}r@{, }r@{, }r@{, }r@{$) \to ($}r@{, }r@{, }r@{, }r@{$)$}@{}}
  \toprule
  $x_{i_1}$ & $x_{i_2}$ & $x_{i_3}$ & $y$ & \multicolumn{9}{c}{Effect on $(w_{i_1}, w_{i_2},w_{i_3}, v)$} \\ \midrule
  $1$ & $-1$ & $0$  & $\tfrac78$ & $($ & $-1$ & $0$ & $1$& $1$  & $-1$ & $0$ & $1$ & $1$  \\
  \multicolumn{4}{c}{} & $($ & $1$ & $0$ & $1$ & $1$ &  $\tfrac34$& $\tfrac14$& $1$  & $\tfrac34$ \\ \midrule
  $0$ & $0$ & $1$ & $\tfrac54$ & $($ & $-1$ & $0$ & $1$ & $1$ &  $-1$ & $0$ & $\tfrac32$& $\tfrac32$ \\
  \multicolumn{4}{c}{} & $($ & $\tfrac34$& $\tfrac14$& $1$  & $\tfrac34$ & $\tfrac34$ & $\tfrac14$ & $\tfrac74$ & $\tfrac74$\\ \midrule
  $0$ & $0$ & $1$  & $\tfrac{119}{16}$ & $($ & $-1$ & $0$ & $\tfrac32$& $\tfrac32$  &  $-1$ & $0$ & $\tfrac{273}{16}$& $\tfrac{273}{16}$  \\
  \multicolumn{4}{c}{} & $($ & $\tfrac34$ & $\tfrac14$ & $\tfrac74$ & $\tfrac74$ & $\tfrac34$ & $\tfrac14$ & $\tfrac{273}{16}$ & $\tfrac{273}{16}$\\ \midrule
  $0$ & $0$ & $\tfrac{16}{273}$  & $\tfrac{289}{32}$ & $($ & $-1$ & $0$ & $\tfrac{273}{16}$& $\tfrac{273}{16}$  & $-1$ & $0$ & $1$& $1$ \\
  \multicolumn{4}{c}{} & $($ & $\tfrac34$ & $\tfrac14$ & $\tfrac{273}{16}$ & $\tfrac{273}{16}$ & $\tfrac34$ & $\tfrac14$ & $1$ & $1$ \\ \midrule
  $-1$ & $1$ & $0$  & $\tfrac78$ & $($ & $-1$ & $0$ & $1$& $1$  &  $-\tfrac34$ & $-\tfrac14$ & $1$& $\tfrac34$ \\
  \multicolumn{4}{c}{} & $($ & $\tfrac34$ & $\tfrac14$ & $1$ & $1$ & $\tfrac34$ & $\tfrac14$ & $1$ & $1$  \\ \midrule
  $0$ & $0$ & $1$  & $\tfrac54$ & $($ & $-\tfrac34$ & $-\tfrac14$ & $1$& $\tfrac34$  & $-\tfrac34$ & $-\tfrac14$ & $\tfrac74$& $\tfrac74$  \\
  \multicolumn{4}{c}{} & $($ & $\tfrac34$ & $\tfrac14$ & $1$ & $1$ &  $\tfrac34$ & $\tfrac14$ & $\tfrac32$ & $\tfrac32$\\ \midrule
  $0$ & $0$ & $1$  & $\tfrac{119}{16}$ & $($ & $-\tfrac34$ & $-\tfrac14$ & $\tfrac74$& $\tfrac74$  &  $-\tfrac34$ & $-\tfrac14$ & $\tfrac{273}{16}$& $\tfrac{273}{16}$\\
  \multicolumn{4}{c}{} & $($ & $\tfrac34$ & $\tfrac14$ & $\tfrac32$ & $\tfrac32$ &$\tfrac34$ & $\tfrac14$ & $\tfrac{273}{16}$ & $\tfrac{273}{16}$ \\ \midrule
  $0$ & $0$ & $\tfrac{16}{273}$  & $\tfrac{289}{32}$ & $($ & $-\tfrac34$ & $-\tfrac14$ & $\tfrac{273}{16}$& $\tfrac{273}{16}$  &  $-\tfrac34$ & $-\tfrac14$ & $1$& $1$ \\
  \multicolumn{4}{c}{} & $($ & $\tfrac34$ & $\tfrac14$ & $\tfrac{273}{16}$ & $\tfrac{273}{16}$ & $\tfrac34$ & $\tfrac14$ & $1$ & $1$ \\ \midrule
  $-1$ & $0$ & $0$  & $\tfrac78$ & $($ & $-\tfrac34$ & $-\tfrac14$ & $1$& $1$  & $-1$ & $-\tfrac14$ & $1$& $\tfrac{19}{16}$ \\
  \multicolumn{4}{c}{} & $($ & $\tfrac34$ & $\tfrac14$ & $1$ & $1$ & $\tfrac34$ & $\tfrac14$ & $1$ & $1$ \\ \midrule
  $0$ & $0$ & $1$  & $\tfrac{27}{16}$ & $($ & $-1$ & $-\tfrac14$ & $1$& $\tfrac{19}{16}$  & $-1$ & $-\tfrac14$ & $\tfrac{35}{16}$& $\tfrac{35}{16}$ \\
  \multicolumn{4}{c}{} & $($ & $\tfrac34$ & $\tfrac14$ & $1$ & $1$ & $\tfrac34$ & $\tfrac14$ & $\tfrac{19}{8}$ & $\tfrac{19}{8}$ \\ \midrule
  $0$ & $0$ & $1$  & $\tfrac{3871}{256}$ & $($ & $-1$ & $-\tfrac14$ & $\tfrac{35}{16}$& $\tfrac{35}{16}$  & $-1$ & $-\tfrac14$ & $\rho$& $\rho$  \\
  \multicolumn{4}{c}{} & $($ & $\tfrac34$ & $\tfrac14$ & $\tfrac{19}{8}$ & $\tfrac{19}{8}$ & $\tfrac34$ & $\tfrac14$ & $\rho$ & $\rho$\\ \midrule
  $0$ & $0$ & $\tfrac{1}{\rho}$  & $\tfrac{\rho+1}{2}$ & $($ & $-1$ & $-\tfrac14$ & $\rho$& $\rho$  &  $-1$ & $-\tfrac14$ & $1$& $1$\\
  \multicolumn{4}{c}{} & $($ & $\tfrac34$ & $\tfrac14$ & $\rho$ & $\rho$ & $\tfrac34$ & $\tfrac14$ & $1$ & $1$ \\ \midrule
  $1$ & $0$ & $0$  & $\tfrac78$ & $($ & $-1$ & $-\tfrac14$ & $1$& $1$  &  $-1$ & $-\tfrac14$ & $1$& $1$ \\
  \multicolumn{4}{c}{} & $($ & $\tfrac34$ & $\tfrac14$ & $1$ & $1$ & $1$ & $\tfrac14$ & $1$ & $\tfrac{19}{16}$ \\ \midrule
  $0$ & $0$ & $1$  & $\tfrac{27}{16}$ & $($ & $-1$ & $-\tfrac14$ & $1$& $1$  &   $-1$ & $-\tfrac14$ & $\tfrac{19}{8}$& $\tfrac{19}{8}$ \\
  \multicolumn{4}{c}{} & $($ & $1$ & $\tfrac14$ & $1$ & $\tfrac{19}{16}$ & $1$ & $\tfrac14$ & $\tfrac{35}{16}$ & $\tfrac{35}{16}$\\ \midrule
  $0$ & $0$ & $1$  & $\tfrac{3871}{256}$ & $($ & $-1$ & $-\tfrac14$ & $\tfrac{19}{8}$& $\tfrac{19}{8}$  &  $-1$ & $-\tfrac14$ & $\rho$& $\rho$  \\
  \multicolumn{4}{c}{} & $($ & $1$ & $\tfrac14$ & $\tfrac{35}{16}$ & $\tfrac{35}{16}$ & $1$ & $\tfrac14$ & $\rho$& $\rho$\\ \midrule
  $0$ & $0$ & $\tfrac{1}{\rho}$  & $\tfrac{\rho+1}{2}$ & $($ & $-1$ & $-\tfrac14$ & $\rho$& $\rho$  & $-1$ & $-\tfrac14$ & $1$& $1$ \\
  \multicolumn{4}{c}{} & $($ & $1$ & $\tfrac14$ & $\rho$& $\rho$ & $1$ & $\tfrac14$ & $1$& $1$ \\ \midrule
  $0$ & $-1$ & $0$  & $\tfrac58$ & $($ & $-1$ & $-\tfrac14$ & $1$& $1$  & $-1$ & $-1$ & $1$& $\tfrac{19}{16}$\\
  \multicolumn{4}{c}{} & $($ & $1$ & $\tfrac14$ & $1$& $1$ &$1$ & $\tfrac14$ & $1$ & $1$ \\ \midrule
   \bottomrule
\end{tabular}
\end{table}

\begin{table}
\small
\caption{Continuing \hyperref[tab:copy2]{Table~\ref{tab:copy2}} (Part~2~of~2) $\left(\text{here } \rho=47\tfrac{1272583}{3125000}\right)$.}%\rdnand{i_1}{i_2}{i_3}{\epsilon_1}{\epsilon_2}.}
%\label{tab:rdnand}
\centering
%\begin{tabular}{@{}cccccr@{}r@{, }r@{, }r@{$) \to ($}r@{, }r@{, }r@{$)$}@{}}
\begin{tabular}{@{}ccccr@{}r@{, }r@{, }r@{, }r@{$) \to ($}r@{, }r@{, }r@{, }r@{$)$}@{}}
  \toprule
  $x_{i_1}$ & $x_{i_2}$ & $x_{i_3}$ & $y$ & \multicolumn{9}{c}{Effect on $(w_{i_1}, w_{i_2},w_{i_3},v)$} \\ \midrule
  $0$ & $0$ & $1$  & $\tfrac{27}{16}$ & $($ & $-1$ & $-1$ & $1$& $\tfrac{19}{16}$  &$-1$ & $-1$ & $\tfrac{35}{16}$ & $\tfrac{35}{16}$ \\
  \multicolumn{4}{c}{} & $($ & $1$ & $\tfrac14$ & $1$ & $1$ & $1$& $\tfrac14$& $\tfrac{19}{8}$  & $\tfrac{19}{8}$ \\ \midrule
  $0$ & $0$ & $1$ & $\tfrac{3871}{256}$ & $($ & $-1$ & $-1$ & $\tfrac{35}{16}$ & $\tfrac{35}{16}$ & $-1$ & $-1$ & $\rho$& $\rho$ \\
  \multicolumn{4}{c}{} & $($ & $1$& $\tfrac14$& $\tfrac{19}{8}$  & $\tfrac{19}{8}$ &  $1$ & $\tfrac14$ & $\rho$ & $\rho$\\ \midrule
  $0$ & $0$ & $\tfrac{1}{\rho}$  & $\tfrac{\rho+1}{2}$ & $($ & $-1$ & $-1$ & $\rho$& $\rho$  & $-1$ & $-1$ & $1$& $1$ \\
  \multicolumn{4}{c}{} & $($ & $1$ & $\tfrac14$ & $\rho$ & $\rho$ &$1$ & $\tfrac14$ & $1$ & $1$  \\ \midrule
  $0$ & $1$ & $0$  & $\tfrac{5}{8}$ & $($ & $-1$ & $-1$ & $1$& $1$  & $-1$ & $-1$ & $1$& $1$  \\
  \multicolumn{4}{c}{} & $($ & $1$ & $\tfrac14$ & $1$ & $1$ &  $1$ & $1$ & $1$ & $\tfrac{19}{16}$ \\ \midrule
  $0$ & $0$ & $1$  & $\tfrac{27}{16}$ & $($ & $-1$ & $-1$ & $1$& $1$  &  $-1$ & $-1$ & $\tfrac{19}{8}$& $\tfrac{19}{8}$ \\
  \multicolumn{4}{c}{} & $($ & $1$ & $1$ & $1$ & $\tfrac{19}{16}$ & $1$ & $1$ & $\tfrac{35}{16}$ & $\tfrac{35}{16}$ \\ \midrule
  $0$ & $0$ & $1$  & $\tfrac{3871}{256}$ & $($ & $-1$ & $-1$ & $\tfrac{19}{8}$& $\tfrac{19}{8}$  &  $-1$ & $-1$ & $\rho$& $\rho$ \\
  \multicolumn{4}{c}{} & $($ & $1$ & $1$ & $\tfrac{35}{16}$ & $\tfrac{35}{16}$ &$1$ & $1$ & $\rho$ & $\rho$  \\ \midrule
  $0$ & $0$ & $\tfrac{1}{\rho}$  & $\tfrac{\rho+1}{2}$ & $($ & $-1$ & $-1$ & $\rho$& $\rho$  & $-1$ & $-1$ & $1$& $1$  \\
  \multicolumn{4}{c}{} & $($ & $1$ & $1$ & $\rho$ & $\rho$ & $1$ & $1$ & $1$ & $1$ \\ \midrule
   \bottomrule
\end{tabular}
\end{table}

\begin{table}
\small
\caption{Training data for \texttt{destructive\_nand}($i_1,i_2,i_3,i_4$) for Dense-ReLU-Dense under Squared Loss. (Part 1 of 3).}
\label{tab:dnand2}
\centering
%\begin{tabular}{@{}cccccr@{}r@{, }r@{, }r@{$) \to ($}r@{, }r@{, }r@{$)$}@{}}
\begin{tabular}{@{}cccccr@{}r@{, }r@{, }r@{, }r@{, }r@{$) \to ($}r@{, }r@{, }r@{, }r@{, }r@{$)$}@{}}
  \toprule
  $x_{i_1}$ & $x_{i_2}$ & $x_{i_3}$ & $x_{i_4}$& $y$ & \multicolumn{11}{c}{Effect on $(w_{i_1}, w_{i_2}, w_{i_3}, w_{i_4}, v)$} \\ \midrule
  $-1$ & $0$ & $0$ & $0$ & $\tfrac{3}{2}$ & $($ & $-1$ & $-1$ & $0$& $1$& $1$& $-2$  & $-1$ & $0$ & $1$& $2$  \\
  \multicolumn{5}{c}{} & $($ & $-1$ & $1$ & $0$& $1$& $1$& $-2$  & $1$ & $0$ & $1$& $2$  \\
  \multicolumn{5}{c}{} & $($ & $1$ & $-1$ & $0$& $1$& $1$& $1$  & $-1$ & $0$ & $1$& $1$  \\
  \multicolumn{5}{c}{} & $($ & $1$ & $1$ & $0$& $1$& $1$& $1$  & $1$ & $0$ & $1$& $1$  \\ \midrule
  $0$ & $0$& $0$ & $1$ & $\tfrac{5}{2}$ & $($ & $-2$ & $-1$ & $0$& $1$& $2$& $-2$  & $-1$ & $0$ & $3$& $3$  \\
  \multicolumn{5}{c}{} & $($ & $-2$ & $1$ & $0$& $1$& $2$& $-2$  & $1$ & $0$ & $3$& $3$  \\
  \multicolumn{5}{c}{} & $($ & $1$ & $-1$ & $0$& $1$& $1$& $1$  & $-1$ & $0$ & $4$& $4$  \\
  \multicolumn{5}{c}{} & $($ & $1$ & $1$ & $0$& $1$& $1$& $1$  & $1$ & $0$ & $4$& $4$  \\ \midrule

  $0$ & $0$& $0$ & $1$ & $\tfrac{73}{2}$ & $($ & $-2$ & $-1$ & $0$& $3$& $3$& $-2$  & $-1$ & $0$ & $168$& $168$  \\
 \multicolumn{5}{c}{} & $($ & $-2$ & $1$ & $0$& $3$& $3$& $-2$  & $1$ & $0$ & $168$& $168$  \\
  \multicolumn{5}{c}{} & $($ & $1$ & $-1$ & $0$& $4$& $4$& $1$  & $-1$ & $0$ & $168$& $168$  \\
  \multicolumn{5}{c}{} & $($ & $1$ & $1$ & $0$& $4$& $4$& $1$  & $1$ & $0$ & $168$& $168$  \\ \midrule

 $0$ & $0$& $0$ & $\tfrac{1}{168}$ & $\tfrac{169}{2}$ & $($ & $-2$ & $-1$ & $0$& $168$& $168$& $-2$  & $-1$ & $0$ & $1$& $1$  \\
 \multicolumn{5}{c}{} & $($ & $-2$ & $1$ & $0$& $168$& $168$& $-2$  & $1$ & $0$ & $1$& $1$  \\
  \multicolumn{5}{c}{} & $($ & $1$ & $-1$ & $0$& $168$& $168$& $1$  & $-1$ & $0$ & $1$& $1$  \\
  \multicolumn{5}{c}{} & $($ & $1$ & $1$ & $0$& $168$& $168$& $1$  & $1$ & $0$ & $1$& $1$  \\ \midrule
$0$ & $-1$& $0$ & $0$ & $\tfrac{3}{2}$ & $($ & $-2$ & $-1$ & $0$& $1$& $1$& $-2$  & $-2$ & $0$ & $1$& $2$  \\
 \multicolumn{5}{c}{} & $($ & $-2$ & $1$ & $0$& $1$& $1$& $-2$  & $1$ & $0$ & $1$& $1$  \\
  \multicolumn{5}{c}{} & $($ & $1$ & $-1$ & $0$& $1$& $1$& $1$  & $-2$ & $0$ & $1$& $2$  \\
  \multicolumn{5}{c}{} & $($ & $1$ & $1$ & $0$& $1$& $1$& $1$  & $1$ & $0$ & $1$& $1$  \\ \midrule
$0$ & $0$& $0$ & $1$ & $\tfrac{5}{2}$ & $($ & $-2$ & $-2$ & $0$& $1$& $2$& $-2$  & $-2$ & $0$ & $3$& $3$  \\
 \multicolumn{5}{c}{} & $($ & $-2$ & $1$ & $0$& $1$& $1$& $-2$  & $1$ & $0$ & $4$& $4$  \\
  \multicolumn{5}{c}{} & $($ & $1$ & $-2$ & $0$& $1$& $2$& $1$  & $-2$ & $0$ & $3$& $3$  \\
  \multicolumn{5}{c}{} & $($ & $1$ & $1$ & $0$& $1$& $1$& $1$  & $1$ & $0$ & $4$& $4$  \\ \midrule
$0$ & $0$& $0$ & $1$ & $\tfrac{73}{2}$ & $($ & $-2$ & $-2$ & $0$& $3$& $3$& $-2$  & $-2$ & $0$ & $168$& $168$  \\
 \multicolumn{5}{c}{} & $($ & $-2$ & $1$ & $0$& $4$& $4$& $-2$  & $1$ & $0$ & $168$& $168$  \\
  \multicolumn{5}{c}{} & $($ & $1$ & $-2$ & $0$& $3$& $3$& $1$  & $-2$ & $0$ & $168$& $168$  \\
  \multicolumn{5}{c}{} & $($ & $1$ & $1$ & $0$& $4$& $4$& $1$  & $1$ & $0$ & $168$& $168$  \\ \midrule
$0$ & $0$& $0$ & $\tfrac{1}{168}$ & $\tfrac{169}{2}$ & $($ & $-2$ & $-2$ & $0$& $168$& $168$& $-2$  & $-2$ & $0$ & $1$& $1$  \\
  \multicolumn{5}{c}{} & $($ & $-2$ & $1$ & $0$& $168$& $168$& $-2$  & $1$ & $0$ & $1$& $1$  \\
  \multicolumn{5}{c}{} & $($ & $1$ & $-2$ & $0$& $168$& $168$& $1$  & $-2$ & $0$ & $1$& $1$  \\
  \multicolumn{5}{c}{} & $($ & $1$ & $1$ & $0$& $168$& $168$& $1$  & $1$ & $0$ & $1$& $1$  \\ \midrule
$1$ & $1$& $1$ & $0$ & $\tfrac{1}{2}$ & $($ & $-2$ & $-2$ & $0$& $1$& $1$& $-2$  & $-2$ & $0$ & $1$& $1$  \\
  \multicolumn{5}{c}{} & $($ & $-2$ & $1$ & $0$& $1$& $1$& $-2$  & $1$ & $0$ & $1$& $1$  \\
  \multicolumn{5}{c}{} &$($ & $1$ & $-2$ & $0$& $1$& $1$& $1$  & $-2$ & $0$ & $1$& $1$  \\
  \multicolumn{5}{c}{} & $($ & $1$ & $1$ & $0$& $1$& $1$& $-2$  & $-2$ & $-3$ & $1$& $-5$  \\ \midrule
$0$ & $0$& $0$ & $1$ & $-\tfrac{9}{4}$ & $($ & $-2$ & $-2$ & $0$& $1$& $1$& $-2$  & $-2$ & $0$ & $-10$& $-10$  \\
 \multicolumn{5}{c}{} & $($ & $-2$ & $1$ & $0$& $1$& $1$& $-2$  & $1$ & $0$ & $-10$& $-10$  \\
  \multicolumn{5}{c}{} & $($ & $1$ & $-2$ & $0$& $1$& $1$& $1$  & $-2$ & $0$ & $-10$& $-10$  \\
  \multicolumn{5}{c}{} & $($ & $-2$ & $-2$ & $-3$& $1$& $-5$& $-2$  & $-2$ & $-3$ & $-4$& $-4$  \\ \midrule

  $0$ & $0$& $0$ & $-1$ & $-\tfrac{311}{2}$ & $($ & $-2$ & $-2$ & $0$& $-10$& $-10$& $-2$  & $-2$ & $0$ & $-1120$& $-1120$  \\
  \multicolumn{5}{c}{} & $($ & $-2$ & $1$ & $0$& $-10$& $-10$& $-2$  & $1$ & $0$ & $-1120$& $-1120$  \\
  \multicolumn{5}{c}{} & $($ & $1$ & $-2$ & $0$& $-10$& $-10$& $1$  & $-2$ & $0$ & $-1120$& $-1120$  \\
  \multicolumn{5}{c}{} & $($ & $-2$ & $-2$ & $-3$& $-4$& $-4$& $-2$  & $-2$ & $-3$ & $-1120$& $-1120$  \\ \midrule

\end{tabular}
\end{table}

\begin{table}
\small
\caption{Continuing  \hyperref[tab:dnand2]{Table~\ref{tab:dnand2}} (Part 2 of 3).}
%\label{tab:rdnand}
\centering
%\begin{tabular}{@{}cccccr@{}r@{, }r@{, }r@{$) \to ($}r@{, }r@{, }r@{$)$}@{}}
\begin{tabular}{@{}cccccr@{}r@{, }r@{, }r@{, }r@{, }r@{$) \to ($}r@{, }r@{, }r@{, }r@{, }r@{$)$}@{}}
  \toprule
  $x_{i_1}$ & $x_{i_2}$ & $x_{i_3}$ & $x_{i_4}$& $y$ & \multicolumn{11}{c}{Effect on $(w_{i_1}, w_{i_2}, w_{i_3}, w_{i_4}, v)$} \\ \midrule
  $0$ & $0$ & $0$ & $-\tfrac{1}{1120}$ & $-\tfrac{1119}{2}$ & $($ & $-2$ & $-2$ & $0$& $-1120$& $-1120$& $-2$  & $-2$ & $0$ & $1$& $1$  \\
  \multicolumn{5}{c}{} & $($ & $-2$ & $1$ & $0$& $-1120$& $-1120$& $-2$  & $1$ & $0$ & $1$& $1$ \\
  \multicolumn{5}{c}{} & $($ & $1$ & $-2$ & $0$& $-1120$& $-1120$& $1$  & $-2$ & $0$ & $1$& $1$ \\
  \multicolumn{5}{c}{} & $($ & $-2$ & $-2$ & $-3$& $-1120$& $-1120$& $-2$  & $-2$ & $-3$ & $1$& $1$ \\ \midrule
  $-1$ & $0$& $0$ & $0$ & $\tfrac{1}{2}$ & $($ & $-2$ & $-2$ & $0$& $1$& $1$& $1$  & $-2$ & $0$ & $1$& $-5$ \\
   \multicolumn{5}{c}{} & $($ & $-2$ & $1$ & $0$& $1$& $1$& $1$  & $1$ & $0$ & $1$& $-5$ \\
  \multicolumn{5}{c}{} & $($ & $1$ & $-2$ & $0$& $1$& $1$& $1$  & $-2$ & $0$ & $1$& $1$ \\
  \multicolumn{5}{c}{} & $($ & $-2$ & $-2$ & $-3$& $1$& $1$& $1$  & $-2$ & $-3$ & $1$& $-5$ \\ \midrule
  $0$ & $0$& $0$ & $1$ & $-\tfrac{9}{2}$ & $($ & $1$ & $-2$ & $0$& $1$& $-5$& $1$  & $-2$ & $0$ & $-4$& $-4$\\
   \multicolumn{5}{c}{} & $($ & $1$ & $1$ & $0$& $1$& $-5$& $1$  & $1$ & $0$ & $-4$& $-4$ \\
  \multicolumn{5}{c}{} & $($ & $1$ & $-2$ & $0$& $1$& $1$& $1$  & $-2$ & $0$ & $-10$& $-10$ \\
  \multicolumn{5}{c}{} & $($ & $1$ & $-2$ & $-3$& $1$& $-5$& $1$  & $-2$ & $-3$ & $-4$& $-4$ \\ \midrule
 $0$ & $0$& $0$ & $-1$ & $-\tfrac{311}{2}$ & $($ & $1$ & $-2$ & $0$& $-4$& $-4$& $1$  & $-2$ & $0$ & $-1120$& $-1120$ \\
  \multicolumn{5}{c}{} & $($ & $1$ & $1$ & $0$& $-4$& $-4$& $1$  & $1$ & $0$ & $-1120$& $-1120$ \\
  \multicolumn{5}{c}{} & $($ & $1$ & $-2$ & $0$& $-10$& $-10$& $1$  & $-2$ & $0$ & $-1120$& $-1120$ \\
  \multicolumn{5}{c}{} & $($ & $1$ & $-2$ & $-3$& $-4$& $-4$& $1$  & $-2$ & $-3$ & $-1120$& $-1120$ \\ \midrule
$0$ & $0$& $0$ & $-\tfrac{1}{1120}$ & $-\tfrac{1119}{2}$ & $($ & $1$ & $-2$ & $0$& $-1120$& $-1120$& $1$  & $-2$ & $0$ & $1$& $1$ \\
  \multicolumn{5}{c}{} & $($ & $1$ & $1$ & $0$& $-1120$& $-1120$& $1$  & $1$ & $0$ & $1$& $1$ \\
  \multicolumn{5}{c}{} & $($ & $1$ & $-2$ & $0$& $-1120$& $-1120$& $1$  & $-2$ & $0$ & $1$& $1$ \\
  \multicolumn{5}{c}{} & $($ & $1$ & $-2$ & $-3$& $-1120$& $-1120$& $1$  & $-2$ & $-3$ & $1$& $1$ \\ \midrule
$1$ & $0$& $0$ & $0$ & $\tfrac{1}{2}$ & $($ & $1$ & $-2$ & $0$& $1$& $1$& $0$  & $-2$ & $0$ & $1$& $0$ \\
   \multicolumn{5}{c}{} & $($ & $1$ & $1$ & $0$& $1$& $1$& $0$  & $1$ & $0$ & $1$& $0$ \\
  \multicolumn{5}{c}{} & $($ & $1$ & $-2$ & $0$& $1$& $1$& $0$  & $-2$ & $0$ & $1$& $0$ \\
  \multicolumn{5}{c}{} & $($ & $1$ & $-2$ & $-3$& $1$& $1$& $0$  & $-2$ & $-3$ & $1$& $0$ \\ \midrule
$0$ & $0$& $0$ & $1$ & $\tfrac{1}{2}$ & $($ & $0$ & $-2$ & $0$& $1$& $0$& $0$  & $-2$ & $0$ & $1$& $1$ \\
  \multicolumn{5}{c}{} & $($ & $0$ & $1$ & $0$& $1$& $0$& $0$  & $1$ & $0$ & $1$& $1$ \\
  \multicolumn{5}{c}{} & $($ & $0$ & $-2$ & $0$& $1$& $0$& $0$  & $-2$ & $0$ & $1$& $1$ \\
  \multicolumn{5}{c}{} & $($ & $0$ & $-2$ & $-3$& $1$& $0$& $0$  & $-2$ & $-3$ & $1$& $1$ \\ \midrule
$0$ & $-1$& $0$ & $0$ & $\tfrac{1}{2}$ & $($ & $0$ & $-2$ & $0$& $1$& $1$& $0$  & $1$ & $0$ & $1$& $-5$ \\
  \multicolumn{5}{c}{} & $($ & $0$ & $1$ & $0$& $1$& $1$& $0$  & $1$ & $0$ & $1$& $1$ \\
  \multicolumn{5}{c}{} & $($ & $0$ & $-2$ & $0$& $1$& $1$& $0$  & $1$ & $0$ & $1$& $-5$ \\
  \multicolumn{5}{c}{} & $($ & $0$ & $-2$ & $-3$& $1$& $1$& $0$  & $1$ & $-3$ & $1$& $-5$ \\ \midrule
$0$ & $0$& $0$ & $1$ & $-\tfrac{9}{2}$ & $($ & $0$ & $1$ & $0$& $1$& $-5$& $0$  & $1$ & $0$ & $-4$& $-4$ \\
  \multicolumn{5}{c}{} & $($ & $0$ & $1$ & $0$& $1$& $1$& $0$  & $1$ & $0$ & $-10$& $-10$ \\
  \multicolumn{5}{c}{} & $($ & $0$ & $1$ & $0$& $1$& $-5$& $0$  & $1$ & $0$ & $-4$& $-4$ \\
  \multicolumn{5}{c}{} & $($ & $0$ & $1$ & $-3$& $1$& $-5$& $0$  & $1$ & $-3$ & $-4$& $-4$ \\ \midrule
$0$ & $0$& $0$ & $-1$ & $-\tfrac{311}{2}$ & $($ & $0$ & $1$ & $0$& $-4$& $-4$& $0$  & $1$ & $0$ & $-1120$& $-1120$ \\
   \multicolumn{5}{c}{} & $($ & $0$ & $1$ & $0$& $-10$& $-10$& $0$  & $1$ & $0$ & $-1120$& $-1120$ \\
  \multicolumn{5}{c}{} & $($ & $0$ & $1$ & $0$& $-4$& $-4$& $0$  & $1$ & $0$ & $-1120$& $-1120$ \\
  \multicolumn{5}{c}{} & $($ & $0$ & $1$ & $-3$& $-4$& $-4$& $0$  & $1$ & $-3$ & $-1120$& $-1120$ \\ \midrule
  $0$ & $0$& $0$ & $-\tfrac{1}{1120}$ & $-\tfrac{1119}{2}$ & $($ & $0$ & $1$ & $0$& $-1120$& $-1120$& $0$  & $1$ & $0$ & $1$& $1$ \\
   \multicolumn{5}{c}{} & $($ & $0$ & $1$ & $0$& $-1120$& $-1120$& $0$  & $1$ & $0$ & $1$& $1$ \\
  \multicolumn{5}{c}{} & $($ & $0$ & $1$ & $0$& $-1120$& $-1120$& $0$  & $1$ & $0$ & $1$& $1$ \\
  \multicolumn{5}{c}{} & $($ & $0$ & $1$ & $-3$& $-1120$& $-1120$& $0$  & $1$ & $-3$ & $1$& $1$ \\ \midrule

\end{tabular}
\end{table}

\begin{table}
\caption{Continuing \hyperref[tab:dnand2]{Table~\ref{tab:dnand2}} (Part 3 of 3).}
%\label{tab:rdnand}
\centering
%\begin{tabular}{@{}cccccr@{}r@{, }r@{, }r@{$) \to ($}r@{, }r@{, }r@{$)$}@{}}
%\begin{tabular}{@{}cccccr@{}r@{, }r@{, }r@{, }r@{$) \to ($}r@{, }r@{, }r@{, }r@{$)$}@{}}
\begin{tabular}{@{}cccccr@{}r@{, }r@{, }r@{, }r@{, }r@{$) \to ($}r@{, }r@{, }r@{, }r@{, }r@{$)$}@{}}
  \toprule
  $x_{i_1}$ & $x_{i_2}$ & $x_{i_3}$ & $x_{i_4}$& $y$ & \multicolumn{11}{c}{Effect on $(w_{i_1}, w_{i_2}, w_{i_3}, w_{i_4}, v)$} \\ \midrule
  $0$ & $1$ & $0$ & $0$ & $\tfrac{1}{2}$ & $($ & $0$ & $1$ & $0$& $1$& $1$& $0$  & $0$ & $0$ & $1$& $0$  \\
  \multicolumn{5}{c}{} & $($ & $0$ & $1$ & $0$& $1$& $1$& $0$  & $0$ & $0$ & $1$& $0$ \\
  \multicolumn{5}{c}{} & $($ & $0$ & $1$ & $0$& $1$& $1$& $0$  & $0$ & $0$ & $1$& $0$ \\
  \multicolumn{5}{c}{} & $($ & $0$ & $1$ & $-3$& $1$& $1$& $0$  & $0$ & $-3$ & $1$& $0$ \\ \midrule
  $0$ & $0$& $0$ & $1$ & $\tfrac{1}{2}$ & $($ & $0$ & $0$ & $0$& $1$& $0$& $0$  & $0$ & $0$ & $1$& $1$ \\
  \multicolumn{5}{c}{} & $($ & $0$ & $0$ & $0$& $1$& $0$& $0$  & $0$ & $0$ & $1$& $1$ \\
  \multicolumn{5}{c}{} & $($ & $0$ & $0$ & $0$& $1$& $0$& $0$  & $0$ & $0$ & $1$& $1$ \\
  \multicolumn{5}{c}{} & $($ & $0$ & $0$ & $-3$& $1$& $0$& $0$  & $0$ & $-3$ & $1$& $1$ \\ \midrule
  $0$ & $0$& $1$ & $1$ & $-\tfrac{3}{2}$ & $($ & $0$ & $0$ & $0$& $1$& $1$& $0$  & $0$ & $-5$ & $-4$& $-4$ \\
  \multicolumn{5}{c}{} & $($ & $0$ & $0$ & $0$& $1$& $1$& $0$  & $0$ & $-5$ & $-4$& $-4$ \\
  \multicolumn{5}{c}{} & $($ & $0$ & $0$ & $0$& $1$& $1$& $0$  & $0$ & $-5$ & $-4$& $-4$ \\
  \multicolumn{5}{c}{} & $($ & $0$ & $0$ & $-3$& $1$& $1$& $0$  & $0$ & $-3$ & $1$& $1$ \\ \midrule
 $0$ & $0$& $0$ & $-\tfrac{1}{4}$ & $-\tfrac{3}{2}$ & $($ & $0$ & $0$ & $-5$& $-4$& $-4$& $0$  & $0$ & $-5$ & $1$& $1$ \\
 \multicolumn{5}{c}{} & $($ & $0$ & $0$ & $-5$& $-4$& $-4$& $0$  & $0$ & $-5$ & $1$& $1$ \\
  \multicolumn{5}{c}{} & $($ & $0$ & $0$ & $-5$& $-4$& $-4$& $0$  & $0$ & $-5$ & $1$& $1$ \\
  \multicolumn{5}{c}{} & $($ & $0$ & $0$ & $-3$& $1$& $1$& $0$  & $0$ & $-3$ & $1$& $1$ \\ \midrule
$0$ & $0$& $-1$ & $0$ & $2$ & $($ & $0$ & $0$ & $-5$& $1$& $1$& $0$  & $0$ & $1$ & $1$& $-29$ \\
  \multicolumn{5}{c}{} & $($ & $0$ & $0$ & $-5$& $1$& $1$& $0$  & $0$ & $1$ & $1$& $-29$ \\
  \multicolumn{5}{c}{} & $($ & $0$ & $0$ & $-5$& $1$& $1$& $0$  & $0$ & $1$ & $1$& $-29$ \\
  \multicolumn{5}{c}{} & $($ & $0$ & $0$ & $-3$& $1$& $1$& $0$  & $0$ & $-1$ & $1$& $-5$ \\ \midrule
$0$ & $0$& $0$ & $1$ & $-\tfrac{57}{2}$ & $($ & $0$ & $0$ & $1$& $1$& $-29$& $0$  & $0$ & $1$ & $-28$& $-28$ \\
 \multicolumn{5}{c}{} & $($ & $0$ & $0$ & $1$& $1$& $-29$& $0$  & $0$ & $1$ & $-28$& $-28$ \\
  \multicolumn{5}{c}{} & $($ & $0$ & $0$ & $1$& $1$& $-29$& $0$  & $0$ & $1$ & $-28$& $-28$ \\
  \multicolumn{5}{c}{} & $($ & $0$ & $0$ & $-1$& $1$& $-5$& $0$  & $0$ & $-1$ & $236$& $-52$ \\ \midrule
$0$ & $0$& $0$ & $1$ & $-\tfrac{24543}{2}$ & $($ & $0$ & $0$ & $1$& $-28$& $-28$& $0$  & $0$ & $1$ & $-28$& $-28$ \\
  \multicolumn{5}{c}{} & $($ & $0$ & $0$ & $1$& $-28$& $-28$& $0$  & $0$ & $1$ & $-28$& $-28$ \\
  \multicolumn{5}{c}{} & $($ & $0$ & $0$ & $1$& $-28$& $-28$& $0$  & $0$ & $1$ & $-28$& $-28$ \\
  \multicolumn{5}{c}{} & $($ & $0$ & $0$ & $-1$& $236$& $-52$& $0$  & $0$ & $-1$ & $184$& $184$ \\ \midrule
$0$ & $0$& $0$ & $\tfrac{1}{184}$ & $78$ & $($ & $0$ & $0$ & $1$& $-28$& $-28$& $0$  & $0$ & $1$ & $-28$& $-28$ \\
  \multicolumn{5}{c}{} & $($ & $0$ & $0$ & $1$& $-28$& $-28$& $0$  & $0$ & $1$ & $-28$& $-28$ \\
  \multicolumn{5}{c}{} & $($ & $0$ & $0$ & $1$& $-28$& $-28$& $0$  & $0$ & $1$ & $-28$& $-28$ \\
  \multicolumn{5}{c}{} & $($ & $0$ & $0$ & $-1$& $184$& $184$& $0$  & $0$ & $-1$ & $-28$& $-28$ \\ \midrule
$0$ & $0$& $0$ & $-\tfrac{1}{28}$ & $-\tfrac{27}{2}$ & $($ & $0$ & $0$ & $1$& $-28$& $-28$& $0$  & $0$ & $1$ & $1$& $1$ \\
 \multicolumn{5}{c}{} & $($ & $0$ & $0$ & $1$& $-28$& $-28$& $0$  & $0$ & $1$ & $1$& $1$ \\
  \multicolumn{5}{c}{} & $($ & $0$ & $0$ & $1$& $-28$& $-28$&$0$  & $0$ & $1$ & $1$& $1$ \\
  \multicolumn{5}{c}{} & $($ & $0$ & $0$ & $-1$& $-28$& $-28$& $0$  & $0$ & $-1$ & $1$& $1$ \\ \midrule

   \bottomrule
\end{tabular}
\end{table}

\begin{table}
\caption{Training data for {\texttt{set\_false\_if\_unset}}($i_1,i_2$) for Dense-ReLU-Dense under Squared Loss.}%\rsetiftrue{i_1}{i_2}{\epsilon_1}.}
\label{tab:setF2}
\centering
\begin{tabular}{@{}cccr@{}r@{, }r@{, }r@{$) \to ($}r@{, }r@{, }r@{$)$}@{}}
  \toprule
  $x_{i_1}$ & $x_{i_2}$ & $y$ & \multicolumn{7}{c}{Effect on $(w_{i_1}, w_{i_2}, v)$} \\ \midrule
  $1$ & $\tfrac12$ & $0$ & $($ & $-1$ & $1$ & $1$&  $-1$ & $1$ & $1$\\
    \multicolumn{3}{c}{} & $($ & $0$& $1$& $1$ & $-1$& $\tfrac12$& $\tfrac12$ \\
  \multicolumn{3}{c}{} & $($ & $1$& $1$& $1$ & $-2$& $-\tfrac12$& $-\tfrac72$\\ \midrule
  $0$ & $-1$ & $-\tfrac{9}{4}$ & $($ & $-1$ & $1$ & $1$&$-1$ & $1$ & $1$ \\
    \multicolumn{3}{c}{} & $($ & $-1$& $\tfrac12$& $\tfrac12$ &  $-1$& $\tfrac12$& $\tfrac12$ \\
  \multicolumn{3}{c}{} & $($ & $-2$& $-\tfrac12$& $-\tfrac72$ & $-2$& $-4$& $-4$ \\ \midrule
  $0$ & $1$ & $\tfrac{5}{4}$ & $($ & $-1$ & $1$ & $1$& $-1$ & $\tfrac32$ & $\tfrac32$\\
    \multicolumn{3}{c}{} & $($ & $-1$& $\tfrac12$& $\tfrac12$ & $-1$ & $\tfrac32$ & $\tfrac32$\\
  \multicolumn{3}{c}{} & $($ & $-2$& $-4$& $-4$ & $-2$& $-4$& $-4$ \\ \midrule
  $0$ & $-1$ & $-\tfrac{245}{16}$ & $($ & $-1$ & $\tfrac32$ & $\tfrac32$& $-1$ & $\tfrac32$ & $\tfrac32$ \\
    \multicolumn{3}{c}{} & $($ & $-1$ & $\tfrac32$ & $\tfrac32$ & $-1$ & $\tfrac32$ & $\tfrac32$ \\
  \multicolumn{3}{c}{} & $($ & $-2$& $-4$& $-4$ & $-2$ & $\tfrac32$ & $\tfrac32$ \\ \midrule
  $0$ & $\tfrac23$ & $\tfrac54$ & $($ & $-1$ & $\tfrac32$ & $\tfrac32$& $-1$ & $1$ & $1$ \\
    \multicolumn{3}{c}{} & $($ & $-1$ & $\tfrac32$ & $\tfrac32$ &$-1$& $1$& $1$  \\
  \multicolumn{3}{c}{} & $($ & $-2$ & $\tfrac32$ & $\tfrac32$ & $-2$& $1$& $1$ \\ \midrule
  $-1$ & $0$ & $\tfrac34$ & $($ & $-1$ & $1$ & $1$& $-\tfrac12$ & $1$ & $\tfrac12$\\
    \multicolumn{3}{c}{} & $($ & $-1$& $1$& $1$ &$-\tfrac12$& $1$& $\tfrac12$\\
  \multicolumn{3}{c}{} & $($ & $-2$& $1$& $1$ &  $\tfrac12$& $1$& $-4$\\ \midrule
  $0$ & $1$ & $1$ & $($ & $-\tfrac12$ & $1$ & $\tfrac12$&  $-\tfrac12$ & $\tfrac32$ & $\tfrac32$\\
    \multicolumn{3}{c}{} & $($ & $-\tfrac12$& $1$& $\tfrac12$ & $-\tfrac12$ & $\tfrac32$ & $\tfrac32$\\
  \multicolumn{3}{c}{} & $($ & $\tfrac12$& $1$& $-4$ & $\tfrac12$& $-39$& $6$\\ \midrule
  $0$ & $-1$ & $\tfrac{467}{2}$ & $($ & $-\tfrac12$ & $\tfrac32$ & $\tfrac32$& $-\tfrac12$ & $\tfrac32$ & $\tfrac32$ \\
    \multicolumn{3}{c}{} & $($ & $-\tfrac12$ & $\tfrac32$ & $\tfrac32$ &  $-\tfrac12$& $\tfrac32$& $\tfrac32$\\
  \multicolumn{3}{c}{} & $($ & $\tfrac12$& $-39$& $6$ &  $\tfrac12$& $-33$& $-33$\\ \midrule
  $0$ & $-1$ & $\tfrac{47893}{44}$ & $($ & $-\tfrac12$ & $\tfrac32$ & $\tfrac32$& $-\tfrac12$ & $\tfrac32$ & $\tfrac32$ \\
    \multicolumn{3}{c}{} & $($ & $-\tfrac12$& $\tfrac32$& $\tfrac32$ &$-\tfrac12$& $\tfrac32$& $\tfrac32$ \\
  \multicolumn{3}{c}{} & $($ & $\tfrac12$& $-33$& $-33$ & $\tfrac12$& $\tfrac32$& $\tfrac32$ \\ \midrule
  $0$ & $\tfrac23$ & $\tfrac54$ & $($ & $-\tfrac12$ & $\tfrac32$ & $\tfrac32$& $-\tfrac12$ & $1$ & $1$ \\
    \multicolumn{3}{c}{} & $($ & $-\tfrac12$& $\tfrac32$& $\tfrac32$ &$-\tfrac12$& $1$& $1$ \\
  \multicolumn{3}{c}{} & $($ & $\tfrac12$& $\tfrac32$& $\tfrac32$ & $\tfrac12$& $1$& $1$ \\ \midrule
  $-1$ & $0$ & $\tfrac34$ & $($ & $-\tfrac12$ & $1$ & $1$&  $-1$ & $1$ & $\tfrac54$\\
    \multicolumn{3}{c}{} & $($ & $-\tfrac12$& $1$& $1$ & $-1$& $1$& $\tfrac54$\\
  \multicolumn{3}{c}{} & $($ & $\tfrac12$& $1$& $1$ &$\tfrac12$& $1$& $1$  \\ \midrule
  $0$ & $1$ & $\tfrac{7}{4}$ & $($ & $-1$ & $1$ & $\tfrac54$& $-1$ & $\tfrac94$ & $\tfrac94$ \\
    \multicolumn{3}{c}{} & $($ & $-1$& $1$& $\tfrac54$ & $-1$ & $\tfrac94$ & $\tfrac94$\\
  \multicolumn{3}{c}{} & $($ & $\tfrac12$& $1$& $1$ & $\tfrac12$& $\tfrac52$& $\tfrac52$\\ \midrule

 % \multicolumn{8}{c}{Return $(\epsilon_1' = \frac{\epsilon_1 \alpha^3}{2}, \epsilon_2 = \alpha^2)$.} \\
  \bottomrule
\end{tabular}
\end{table}

\begin{table}
\caption{Continuing \hyperref[tab:setF2]{Table~\ref{tab:setF2}}.}%\rsetiftrue{i_1}{i_2}{\epsilon_1}.}
%\label{tab:rsetiftrue}
\centering
\begin{tabular}{@{}cccr@{}r@{, }r@{, }r@{$) \to ($}r@{, }r@{, }r@{$)$}@{}}
  \toprule
  $x_{i_1}$ & $x_{i_2}$ & $y$ & \multicolumn{7}{c}{Effect on $(w_{i_1}, w_{i_2}, v)$} \\ \midrule
  $0$ & $1$ & $\tfrac{263}{16}$ & $($ & $-1$ & $\tfrac94$ & $\tfrac94$&$-1$ & $\tfrac{855}{16}$ & $\tfrac{855}{16}$ \\
    \multicolumn{3}{c}{} & $($ & $-1$ & $\tfrac94$ & $\tfrac94$ &  $-1$& $\tfrac{855}{16}$& $\tfrac{855}{16}$\\
  \multicolumn{3}{c}{} & $($ & $\tfrac12$& $\tfrac52$& $\tfrac52$ &$\tfrac12$& $\tfrac{855}{16}$& $\tfrac{855}{16}$\\ \midrule
  $0$ & $\tfrac{16}{855}$ & $\tfrac{871}{32}$ & $($ & $-1$ & $\tfrac{855}{16}$ & $\tfrac{855}{16}$&  $-1$ & $1$ & $1$\\
    \multicolumn{3}{c}{} & $($ & $-1$& $\tfrac{855}{16}$& $\tfrac{855}{16}$ & $-1$& $1$& $1$ \\
  \multicolumn{3}{c}{} & $($ & $\tfrac12$& $\tfrac{855}{16}$& $\tfrac{855}{16}$ &$\tfrac12$& $1$& $1$ \\ \midrule
  $1$ & $0$ & $\tfrac{3}{4}$ & $($ & $-1$ & $1$ & $1$& $-1$ & $1$ & $1$ \\
    \multicolumn{3}{c}{} & $($ & $-1$& $1$& $1$ &$-1$& $1$& $1$   \\
  \multicolumn{3}{c}{} & $($ & $\tfrac12$& $1$& $1$ & $1$& $1$& $\tfrac54$ \\ \midrule
  $0$ & $1$ & $\tfrac{7}{4}$ & $($ & $-1$ & $1$ & $1$&$-1$ & $\tfrac52$ & $\tfrac52$ \\
    \multicolumn{3}{c}{} & $($ & $-1$& $1$& $1$ &  $-1$& $\tfrac52$& $\tfrac52$\\
  \multicolumn{3}{c}{} & $($ & $1$& $1$& $\tfrac54$ & $1$& $\tfrac94$& $\tfrac94$\\ \midrule
  $0$ & $1$ & $\tfrac{263}{16}$ & $($ & $-1$ & $\tfrac52$ & $\tfrac52$& $-1$ & $\tfrac{855}{16}$ & $\tfrac{855}{16}$ \\
    \multicolumn{3}{c}{} & $($ & $-1$& $\tfrac52$& $\tfrac52$ & $-1$& $\tfrac{855}{16}$& $\tfrac{855}{16}$ \\
  \multicolumn{3}{c}{} & $($ & $1$& $\tfrac94$& $\tfrac94$ & $1$& $\tfrac{855}{16}$& $\tfrac{855}{16}$ \\ \midrule
  $0$ & $\tfrac{16}{855}$ & $\tfrac{871}{32}$ & $($ & $-1$ & $\tfrac{855}{16}$ & $\tfrac{855}{16}$& $-1$ & $1$ & $1$ \\
    \multicolumn{3}{c}{} & $($ & $-1$& $\tfrac{855}{16}$& $\tfrac{855}{16}$ & $-1$& $1$& $1$ \\
  \multicolumn{3}{c}{} & $($ & $1$& $\tfrac{855}{16}$& $\tfrac{855}{16}$ & $1$& $1$& $1$ \\ \midrule
 % \multicolumn{8}{c}{Return $(\epsilon_1' = \frac{\epsilon_1 \alpha^3}{2}, \epsilon_2 = \alpha^2)$.} \\
  \bottomrule
\end{tabular}
\end{table}

\begin{table}
\small
\caption{Training data for \texttt{copy\_if\_true}($i_1,i_2,i_3$) for Dense-ReLU-Dense under Squared Loss.}%\rdnand{i_1}{i_2}{i_3}{\epsilon_1}{\epsilon_2}.}
\label{tab:setT2}
\centering
%\begin{tabular}{@{}cccccr@{}r@{, }r@{, }r@{$) \to ($}r@{, }r@{, }r@{$)$}@{}}
\begin{tabular}{@{}ccccr@{}r@{, }r@{, }r@{, }r@{$) \to ($}r@{, }r@{, }r@{, }r@{$)$}@{}}
  \toprule
  $x_{i_1}$ & $x_{i_2}$ & $x_{i_3}$ & $y$ & \multicolumn{9}{c}{Effect on $(w_{i_1}, w_{i_2},w_{i_3}, v)$} \\ \midrule
  $1$ & $-2$ & $0$  & $\tfrac{3}{4}$ & $($ & $-1$ & $0$ & $1$& $1$  &$-1$ & $0$ & $1$ & $1$\\
  \multicolumn{4}{c}{} & $($ & $1$ & $0$ & $1$ & $1$ &$\tfrac12$& $1$& $1$  & $\tfrac12$ \\ \midrule
  $0$ & $0$ & $1$ & $1$ & $($ & $-1$ & $0$ & $1$ & $1$ & $-1$ & $0$ & $1$& $1$ \\
  \multicolumn{4}{c}{} & $($ & $\tfrac12$& $1$& $1$  & $\tfrac12$ & $\tfrac12$ & $1$ & $\tfrac32$ & $\tfrac32$ \\ \midrule
  $0$ & $0$ & $1$  & $\tfrac{17}{4}$ & $($ & $-1$ & $0$ & $1$& $1$  & $-1$ & $0$ & $\tfrac{15}{2}$& $\tfrac{15}{2}$ \\
  \multicolumn{4}{c}{} & $($ & $\tfrac12$ & $1$ & $\tfrac32$ & $\tfrac32$ & $\tfrac12$ & $1$ & $\tfrac{15}{2}$ & $\tfrac{15}{2}$\\ \midrule
  $0$ & $0$ & $\tfrac{2}{15}$  & $\tfrac{17}{4}$ & $($ & $-1$ & $0$ & $\tfrac{15}{2}$& $\tfrac{15}{2}$  & $-1$ & $0$ & $1$& $1$   \\
  \multicolumn{4}{c}{} & $($ & $\tfrac12$ & $1$ & $\tfrac{15}{2}$ & $\tfrac{15}{2}$ &$\tfrac12$ & $1$ & $1$ & $1$   \\ \midrule
  $1$ & $0$ & $0$  & $\tfrac{3}{4}$ & $($ & $-1$ & $0$ & $1$& $1$  &  $-1$ & $0$ & $1$& $1$ \\
  \multicolumn{4}{c}{} & $($ & $\tfrac12$ & $1$ & $1$ & $1$ & $1$ & $1$ & $1$ & $\tfrac54$ \\ \midrule
  $0$ & $0$ & $1$  & $\tfrac{7}{4}$ & $($ & $-1$ & $0$ & $1$& $1$  & $-1$ & $0$ & $\tfrac52$& $\tfrac52$  \\
  \multicolumn{4}{c}{} & $($ & $1$ & $1$ & $1$ & $\tfrac54$ &$1$ & $1$ & $\tfrac94$ & $\tfrac94$ \\ \midrule
  $0$ & $0$ & $1$  & $\tfrac{263}{16}$ & $($ & $-1$ & $0$ & $\tfrac52$& $\tfrac52$  &   $-1$ & $0$ & $\tfrac{855}{16}$& $\tfrac{855}{16}$ \\
  \multicolumn{4}{c}{} & $($ & $1$ & $1$ & $\tfrac94$ & $\tfrac94$ &$1$ & $1$ & $\tfrac{855}{16}$ & $\tfrac{855}{16}$ \\ \midrule
  $0$ & $0$ & $\tfrac{16}{855}$  & $\tfrac{871}{32}$ & $($ & $-1$ & $0$ & $\tfrac{855}{16}$& $\tfrac{855}{16}$  & $-1$ & $0$ & $1$& $1$  \\
  \multicolumn{4}{c}{} & $($ & $1$ & $1$ & $\tfrac{855}{16}$ & $\tfrac{855}{16}$ & $1$ & $1$ & $1$ & $1$ \\ \midrule
   \bottomrule
\end{tabular}
\end{table}

\end{document}